\documentclass[twoside]{article}

\usepackage{multicol}
\usepackage{microtype}
\usepackage{graphicx}
\usepackage{subfigure}
\usepackage{booktabs} 
\usepackage{natbib}
\usepackage{epsfig,epsf,fancybox}
\usepackage{amsmath}
\usepackage{mathrsfs}
\usepackage{amssymb}
\usepackage{color}
\usepackage{multirow}
\usepackage{paralist}
\usepackage{verbatim}
\usepackage{algorithm}
\usepackage{algorithmic}
\usepackage{galois}
\usepackage{boxedminipage}
\usepackage{accents}
\usepackage{stmaryrd}
\usepackage[bottom]{footmisc}
\definecolor{light-gray}{gray}{0.85}
\usepackage{colortbl}
\usepackage{makecell}
\usepackage{hhline}
\usepackage{bbm}
\usepackage{mathtools}
\usepackage{xcolor}         
\usepackage{bbold}
\usepackage[utf8]{inputenc} 
\usepackage[T1]{fontenc}    
\usepackage{hyperref}       
\usepackage{url}            
\usepackage{booktabs}       
\usepackage{amsfonts}       
\usepackage{nicefrac}       
\usepackage{microtype}      
\usepackage{MnSymbol, graphicx}
\usepackage[accepted]{aistats2022}

\newtheorem{theorem}{Theorem}[section]
\newtheorem{corollary}[theorem]{Corollary}
\newtheorem{lemma}[theorem]{Lemma}
\newtheorem{proposition}[theorem]{Proposition}

\newtheorem{remark}[theorem]{Remark}
\newtheorem{assumption}[theorem]{Assumption}

\def\endproof{\hfill $\Box$ \vskip 0.4cm}

\newcommand{\Lamb}{{\boldsymbol \lambda}}
\newcommand{\muu}{{\boldsymbol \mu}}

\newcommand{\argmin}{\mathop{\rm argmin}}

\newcommand{\mc}{\mathcal}
\newcommand{\mbb}{\mathbb}
\newcommand{\mb}{\mathbf}

\newcommand{\ba}{\begin{array}}
\newcommand{\ea}{\end{array}}

\begin{document}

%


%

\twocolumn[

\aistatstitle{A Dual Approach to Constrained Markov Decision Processes with Entropy Regularization}

\aistatsauthor{ Donghao Ying \And Yuhao Ding \And  Javad Lavaei }

\aistatsaddress{ University of California, Berkeley\\\texttt{donghaoy@berkely.edu}\\
\And  University of California, Berkeley\\\texttt{yuhao\_ding@berkely.edu} \And University of California, Berkeley\\\texttt{lavaei@berkeley.edu} } ]

\begin{abstract}
We study entropy-regularized constrained Markov decision processes (CMDPs) under the soft-max parameterization, in which an agent aims to maximize the entropy-regularized value function while satisfying constraints on the expected total utility. By leveraging the entropy regularization, our theoretical analysis shows that its Lagrangian dual function is smooth and the Lagrangian duality gap can be decomposed into the primal optimality gap and the constraint violation. Furthermore, we propose an accelerated dual-descent method for entropy-regularized CMDPs. We prove that our method achieves the global convergence rate $\widetilde{\mathcal{O}}(1/T)$ for both the optimality gap and the constraint violation for entropy-regularized CMDPs. A discussion about a linear convergence rate for CMDPs with a single constraint is also provided. 
\end{abstract}

 \section{INTRODUCTION}


In many sequential decision-making problems for safety-critical systems, e.g. autonomous driving \citep{fisac2018general} and cyber-physical systems \citep{zhang2019non}, the optimality of an objective function by itself is not sufficient and a variety of constraints must be satisfied. 
This has naturally led to a generalization of the model of Markov Decision Processes (MDPs) to Constrained MDPs (CMDPs) \citep{altman1999constrained}, in which an agent aims to maximize the value function while satisfying given constraints on the expected total utility.  

Direct policy search methods, including the policy gradient and the natural policy gradient (NPG) methods, have had substantial empirical successes in solving CMDPs \citep{achiam2017constrained, chow2017risk,bhatnagar2012online,borkar2005actor,uchibe2007constrained, achiam2017constrained}. 
Recently, a major progress in understanding the theoretical non-asymptotic global convergence
behavior of policy-based methods for CMDPs has also been achieved \citep{ding2020natural, ding2021provably, xu2021crpo, efroni2020exploration,chen2021primal}. 

For policy-based methods, entropy regularization is a popular technique for encouraging the exploration of an unknown environment and preventing a premature convergence \citep{williams1991function,mnih2016asynchronous,haarnoja2018soft,zang2020teac}. 
From a theoretical optimization perspective, it is shown in \cite{mei2020global} and \cite{cen2021fast} that the entropy regularization can make the policy optimization landscape benign and achieve faster convergence rates even in the exact value evaluation setting. 
Nevertheless, most existing theoretical guarantees for the entropy-regularized policy optimization are restricted to unconstrained MDPs. The scope of the power of entropy regularization for CMDPs remains unknown even for the tabular setting with the exact value evaluation.

Inspired by the recent theoretical advances towards understanding entropy-regularized policy gradient methods \citep{mei2020global, cen2021fast} together with the global convergence of Lagrangian-based methods for CMDPs \citep{ding2020natural, ding2021provably, paternain2019safe, xu2021crpo}, {we investigate the optimization properties induced by the entropy regularization} for CMDPs under the soft-max policy parameterization. We focus on the study of tabular CMDPs with the exact gradient evaluation. 
This is the setting
commonly investigated in the literature since its understanding assists in 
demystifying the effectiveness of entropy-regularization in CMDPs with more complex settings.

\subsection{Contributions}  
{This work is the first one 
that certifies the effectiveness of entropy regularization in CMDPs from an optimization perspective.}
Our contributions are as follows:
\begin{itemize}
    \item We first show that {although the underlying problem is nonconcave, the Lagrangian dual function of CMDPs with the entropy regularization is smooth under the Slater condition and the exploratory initial distribution assumption.} Under the same conditions, an $\mathcal{O}\left({\varepsilon}\right)$ error bound for the dual optimality gap leads to an $\mathcal{O}\left(\sqrt{\varepsilon}\right)$ error bound for the primal optimality gap and the constraint violation.
    \item {To leverage the smoothness of the Lagrangian dual function}, we propose a new accelerated dual-descent method for entropy-regularized CMDPs, which updates the dual variable via projected accelerated gradient descent and uses the natural policy gradient method in the inner loop.
    \item  We prove that the proposed method achieves a global convergence with the rate $\widetilde{\mathcal{O}}(1/T)$ for both the optimality gap and the constraint violation for entropy-regularized CMDPs. 
    \item {In the special case where CMDPs only have a single constraint, we show that a bisection-based dual method can achieve a linear convergence rate.}
\end{itemize}

\subsection{Related Work}
\paragraph{CMDPs}
Our work is related to policy-based CMDP algorithms \citep{altman1999constrained, borkar2005actor, bhatnagar2012online, chow2017risk,ding2020natural, ding2021provably, xu2021crpo, chen2021primal, efroni2020exploration}. 
The papers \cite{ding2020natural} and \cite{xu2021crpo} are closely related to our work. In  \cite{ding2020natural}, the authors propose a natural policy gradient primal-dual method for CMDPs and prove that
it achieves global convergence with the rate $\mathcal{O}(1/\sqrt{T})$ for both the optimality gap and the constraint violation under the soft-max policy parameterization. The work \cite{xu2021crpo} achieves a similar global convergence rate as \cite{ding2020natural} using a primal-based approach. However, the entropy regularization, which is an effective technique for unconstrained MDPs, is not used in these algorithms.

\paragraph{Entropy-regularized RL}
Maximum entropy reinforcement learning optimizes policies to jointly maximize the expected return and the expected entropy of the policy. This framework has been
used in many contexts. It has been shown that the maximum entropy formulation provides a substantial
improvement in exploration and robustness \citep{ziebart2010modeling}. 
It is robust in the face of model and estimation errors \citep{haarnoja2017reinforcement} in both on-policy and off-policy settings \citep{haarnoja2018soft}. More recently, the theoretical results in \cite{mei2020global} and \cite{cen2021fast} have shown that the entropy regularization can help policy-based methods improve the convergence rate and the sample complexity compared with standard MDPs without the entropy regularization.
However, despite the tremendous successes of the entropy regularization in unconstrained MDPs, the impact of the entropy regularization for CMDPs remains unknown.

\subsection{Notations}
{
Let $\Delta(\mathcal{S})$ denote the probability simplex over the set $\mathcal{S}$, and let $\left|\mathcal{S}\right|$ denote its cardinality.
For a set $T\subset \mbb{R}^p$, let $\operatorname{cl}(T)$ denote the closure of $T$.
When the variable $s$ follows the distribution $\rho$, we write it as $s\sim \rho$.
Let $\mbb{E}[\cdot]$ and $\mbb{E}[\cdot\mid \cdot]$, respectively, denote the expectation and conditional expectation of a random variable.
Let $\mbb{R}$ denote the set of real numbers.
For a number $a\in \mathbb{R}$, let $\operatorname{sign}(a)$ denote the sign of $a$, i.e. $\operatorname{sign}(a) = +1$ if $a \geq0$ and $\operatorname{sign}(a) = -1$ if $a <0$.
Let $[n]$ denote the set $\{1,2,\dots, n\}$.
For a vector $x$, we use $x^\top$ to denote the transpose of $x$,
and use $x_i$ or $(x)_i$ to denote the $i$-th entry of $x$.
When applying a scalar function to $x$, e.g. $\log x$, the operation is understood as entry-wise. 
For vectors $x$ and $y$, we use $x \geq y$ to denote an entry-wise inequality.
We use the convention that $\|x\|_1 = \sum_i |x_i|$, $\|x\|_2 = \sqrt{\sum_i x_i^2}$, and $\|x\|_\infty = \max_i |x_i|$.
For a matrix $A$, we use $A_{ij}$ to denote its $(i,j)$-th entry, and let $\|A\|_F = \sqrt{\sum_{i,j}A_{ij}^2}$.
Let $I_n$ denote the $n\times n$ identity matrix.
For square matrices $A$ and $B$, we use $A\succeq B$ to denote that $A-B$ is positive semi-definite.
For a function $f(x)$, let $\nabla_x f(x)$ (resp. $\nabla^2_{xx}f(x)$) denote its gradient (resp. Hessian) with respect to $x$, and we may omit $x$ in the subscript when it is clear from the context.
Let $\operatorname{arg}\min f(x)$ (resp. $\operatorname{arg}\max f(x)$) denote any arbitrary global minimum (resp. global maximum) of $f(x)$.
We use boldface symbols for constraint-related vectors, e.g. $\Lamb$.}

 \section{PROBLEM FORMULATION}\label{sec:formulation}
\paragraph{Markov Decision Processes}
An infinite-horizon Markov Decision Process  MDP$(\mc{S},\mc{A},P,r,\gamma)$ with a finite state-action space is specified by: a finite state space $\mc{S}$; a finite action space $\mc{A}$; a transition dynamics $P: \mc{S}\times \mc{A} \rightarrow \Delta (\mc{S})$, where $P(s^\prime \vert s,a)$ is the probability of transition from state $s$ to state $s^\prime$ when action $a$ is taken; a reward function $r: \mc{S}\times \mc{A} \rightarrow [0,1]$, where $r(s,a)$ is the instantaneous reward when taking action $a$ in state $s$; a discount factor $\gamma\in[0,1)$. 
A policy $\pi: \mc{S}\rightarrow \Delta (\mc{A})$ represents that the decision rule the agent uses, i.e. the agent takes action $a$ with probability $\pi(a\vert s)$ in state $s$.
We can also interpret a policy $\pi$ as a vector in $\Delta(\mathcal{A})^{|\mathcal{S}|}\subset \mbb{R}^{|\mc{S}||\mc{A}|}$. 

Given a policy $\pi$, the value function $V^\pi : \mc{S} \rightarrow \mathbb{R}$ is defined to characterize the discounted sum of the rewards earned under $\pi$, i.e.
\begin{equation}
V^{\pi}(s):=\mbb{E}\left[\sum_{t=0}^{\infty} \gamma^{t} r\left(s_{t}, a_{t}\right) \bigg| \pi, s_{0}=s\right],\ \forall s\in\mc{S}
\end{equation}
where the expectation is taken over all possible trajectories, in which $a_t \sim \pi(\cdot \vert s_t)$ and $s_{t+1} \sim P(\cdot \vert s_t, a_t)$.
When the initial state is sampled from some distribution $\rho$, we slightly abuse the notation and define the value function as
\begin{equation}
V^{\pi}(\rho):=\mbb{E}_{s\sim \rho}\left[V^{\pi}(s)\right]
\end{equation}
One classical property of the value function is that it is sufficiently smooth with respect to the policy if we view $V^\pi(\rho)$ as a function from the policy space $\Delta(\mathcal{A})^{|\mathcal{S}|}$ to $\mbb{R}$.
Especially, $V^\pi(\rho)$ has the following Lipschitz property.
\begin{lemma}\label{lemma:lipschitz value function}
For arbitrary policies $\pi_1$ and $\pi_2$, it holds
\begin{equation}\label{eq:lipschitz value function}
\left|V^{\pi_1}(\rho) -V^{\pi_2}(\rho) \right| \leq \ell_c \left\|\pi_1 - \pi_2\right\|_2,
\end{equation}
where $\ell_c = \frac{\sqrt{|\mathcal{A}|}}{(1-\gamma)^2}$.
\end{lemma}
Lemma \ref{lemma:lipschitz value function} follows from the bounded gradient of $V^\pi(\rho)$ with respect to $\pi$.
We refer the reader to the supplement in Appendix \ref{app:1} for more details.

The action-value function (or Q-function) $Q^\pi:\mc{S}\times \mc{A}\rightarrow \mbb{R}$ under policy $\pi$ is defined as
\begin{equation}
Q^{\pi}(s, a)=\mathbb{E}\left[\sum_{t=0}^{\infty} \gamma^{t} r\left(s_{t}, a_{t}\right) \bigg| \pi, s_{0}=s, a_{0}=a\right]
\end{equation}
which can be interpreted as the expected total reward with an initial state $s_0 = s$ and an initial action $a_0 = a$.
Since $r(s,a)\in [0,1]$ by assumption, we have that both $Q^{\pi}(s, a)$ and $V^{\pi}(\rho)$ are bounded between $\left[0,  {1}/{(1-\gamma)}\right]$ for any $(s,a)\in \mc{S}\times \mc{A}$ and initial distribution $\rho$.

For theoretical analysis, it is useful to define the so-called discounted state visitation distribution $d_{s_0}^\pi$ of a policy $\pi$:
\begin{equation}\label{eq:visitation measure}
d_{s_{0}}^{\pi}(s):=(1-\gamma) \sum_{t=0}^{\infty} \gamma^{t} {P}\left(s_{t}=s \mid \pi, s_{0}\right),\ \forall s\in\mc{S}
\end{equation}
and we write $d_\rho^\pi(s) := \mbb{E}_{s_0\sim \rho} \left[d^\pi_{s_0}(s)\right]$ as the visitation distribution when the initial state follows $\rho$.


\paragraph{Soft-max Parameterization}
Parameterization is commonly  deployed to model unknown policies to help with the optimization process. 
One natural choice is the soft-max parameterization:
\begin{equation}\label{eq:softmax}
\pi_\theta (a\vert s):= \frac{\exp (\theta(s, a))}{\sum_{a^{\prime} \in \mathcal{A}} \exp \left(\theta\left(s, a^{\prime}\right)\right)},\ \forall (s,a)\in \mc{S}\times\mc{A}
\end{equation}
where $\theta \in \mbb{R}^{|\mc{S}||\mc{A}|}$ is an unconstrained vector.
We denote the class of all soft-max parameterized policies by $\Pi$.
This policy class is complete in the sense that its closure cl$(\Pi)$ contains all stationary policies.
In what follows, we will discard the subscript $\theta$ and just write $\pi \in \Pi$, whenever it is clear from the context.

\paragraph{Entropy Regularization}
To encourage exploration and accelerate convergence to the optimal policy, entropy regularization is widely used in solving MDPs.
In the regularized setting, the agent seeks to optimize the entropy-regularized value function
\begin{equation}
V_{\tau}^{\pi}(\rho):=V^{\pi}(\rho)+\tau \cdot \mathcal{H}(\rho, \pi),
\end{equation}
where $\tau \geq 0$ specifies the weight of regularization and $\mc{H}(\rho, \pi)$ is the discounted entropy defined by
\begin{equation}
\mathcal{H}(\rho, \pi):={\mathbb{E}}\left[\sum_{t=0}^{\infty}-\gamma^{t} \log \pi\left(a_{t} \vert s_{t}\right) \bigg|\pi, s_0 \sim \rho \right].
\end{equation}
We can also define the Q-function under regularization, which is referred to as the soft Q-function
\begin{equation}\label{def:soft Q function}
Q_{\tau}^{\pi}(s, a)=r(s, a)+\gamma \mathbb{E}_{s^{\prime} \sim P(\cdot \mid s, a)}\left[V_{\tau}^{\pi}\left(s^{\prime}\right)\right].
\end{equation}

\paragraph{Constrained MDP With Entropy Regularization}
In a Constrained Markov Decision Process CMDP$(\mc{S},\mc{A},P,r,\mb{g},\mb{b},\gamma)$, besides the reward function $r$, we have a utility function $\mb{g} = (g_1,\dots,g_n): \mc{S}\times \mc{A} \rightarrow [0,1]^n $ and a threshold $\mb{b}\in \left[0,{1}/{(1-\gamma)}\right]^n$.
Under entropy regularization, the agent seeks to maximize the regularized value function $V_{\tau}^\pi(\rho)$ while satisfying the constraint $U_\mb{g}^\pi (\rho) \geq \mb{b}$, where the discounted utility $U_\mb{g}^\pi (\rho) := \left(U^\pi_{g_1}(\rho), \dots, U^\pi_{g_n}(\rho)\right) \in \mbb{R}^n$ is defined by
\begin{equation}
U^\pi_{g_i}(\rho):= \mbb{E}\left[\sum_{t=0}^{\infty} \gamma^t g_i(s_t,a_t)\bigg|\pi, s_0 \sim \rho \right].
\end{equation}
Equivalently, the agent solves the optimization problem
\begin{equation}\label{prob:maxentrl}
\max_{\pi\in \Pi}  \ V_{\tau}^\pi(\rho) \quad \text{s.t.} \ U_\mb{g}^\pi (\rho) \geq \mb{b},
\end{equation}
where $V_{\tau}^\pi(\rho):= V^\pi(\rho) + \tau \mc{H}(\rho,\pi)$. 
Consider the associated Lagrangian function $L(\pi,\Lamb)$ and the dual function $D(\Lamb)$ defined as:
\begin{subequations}
\begin{align}
L(\pi,\Lamb) &:= V_{\tau}^\pi(\rho) + \Lamb^\top\left(U_\mb{g}^\pi (\rho)-\mb{b}\right), \label{def:lagrangian}\\
D(\Lamb)&:= \max_{\pi\in \Pi} L(\pi,\Lamb), \label{def:dual function}
\end{align}
\end{subequations}
where $\Lamb \in \mathbb{R}^n$ is the dual variable.
For brevity, we omit the dependency of $L$ and $D$ on $\rho$ and  $\tau$ in the notations.
It can be seen from (\ref{def:lagrangian}) that $L(\pi,\Lamb)$ can be viewed as an entropy-regularized value function with the reward $r_{\Lamb} (s,a) :=r(s,a) + \Lamb^\top\mb{g}(s,a)$ subtracted by the scalar $\Lamb^\top\mb{b}$.
Let $V_{\Lamb}^\pi (\rho)$ (resp. $Q_{\Lamb} ^\pi (s,a)$) denote the value function (resp. Q-function) with the reward function $r_\Lamb$, i.e. $V_{\Lamb}^\pi (\rho):= V_{\tau}^\pi(\rho) + \Lamb^\top U_\mb{g}^\pi (\rho)$.


It is worth mentioning that (\ref{prob:maxentrl}) is non-convex due to its non-concave objective function and non-convex constraints, thus making the problem challenging to solve.
Henceforth, we slightly abuse the notation and still denote $\pi_\tau^\star$ as an arbitrary optimal policy to the constrained problem (\ref{prob:maxentrl}). 
Let $\Lamb^\star$ denote an optimal multiplier, i.e. 
\begin{equation}
    \Lamb^\star:= \operatorname{arg}\min_{\Lamb\geq 0} D(\Lamb),
\end{equation} 
and $\pi_\Lamb$ be the Lagrangian maximizer associated with the multiplier $\Lamb$, i.e. 
\begin{equation} \label{eq:pi_lambda evaluation}
\pi_\Lamb:= \operatorname{arg}\max_{\pi\in \Pi} L(\pi,\Lamb).
\end{equation} 
We use the shorthand notations $V_\tau^\star := V_{\tau}^{\pi^\star_\tau}(\rho)$ and $D_\tau^\star := \min_{\Lamb\geq 0} D(\Lamb)$.
As before, we hide the dependency of $\Lamb^\star$, $\pi_\Lamb$ on $\rho$ and $\tau$, as well as the dependency of $\pi_\tau^\star$, $V_\tau^\star$, $D_\tau^\star$ on $\rho$.


 \section{PROPERTIES OF CMDP WITH ENTROPY REGULARIZATION}\label{sec:property}
Despite its non-convex nature, entropy-regularized CMDPs enjoy desirable properties, which we will discuss below.
We refer the reader to the supplement in Appendix \ref{app:1} for all the proofs in this section.

Assume that the Slater condition holds, i.e. there exists a strictly feasible policy.
\begin{assumption}[Slater Condition]\label{assump:slater}
There exist a policy $\overline \pi \in \Pi$ and $\boldsymbol \xi >0$ such that $U^{\overline \pi}_\mb{g} (\rho) -\mb{b} \geq \boldsymbol \xi$.
\end{assumption}
The Slater condition is standard in constrained optimization.
It holds when the feasible region contains an interior point.
In practical, such a point is often easy to find given prior knowledge of the problem.
\begin{assumption}[Strong Duality]\label{assump:strong duality}
Suppose there exists a primal-dual pair $\left(\pi_\tau^\star, \Lamb^\star\right)$ such that $V_\tau^\star = D_\tau^\star = L\left(\pi_\tau^\star,\Lamb^\star\right)$.
\end{assumption}
In Appendix \ref{sec:app_discuss}, we provide a discussion for the scenario where the strong duality does not hold.
In the remainder of the paper, we always assume that $\left(\pi_\tau^\star,\Lamb^\star\right)$ is a primal-dual pair.
From the strong duality, we can derive an upper bound on $\Lamb^\star$.
\begin{lemma}\label{lemma:dual boundedness}
Under Assumptions \ref{assump:slater} and \ref{assump:strong duality}, it holds that 
\begin{equation}
0 \leq \lambda^{\star}_i \leq\frac{V_\tau^\star-V_{\tau}^{\overline{\pi}}(\rho)}{\xi_i}, \quad \forall i \in[n].
\end{equation}
\end{lemma}
Define 
\begin{equation}\label{eq:Lambda definition}
\Lambda := \left\{\Lamb \mid 0\leq \lambda_i \leq \frac{V_\tau^\star-V_{\tau}^{\overline{\pi}}(\rho)}{\xi_i},\text{ for all } i \in [n]\right\}.
\end{equation}
Since the dual function $D(\Lamb)$ is always convex, Assumption \ref{assump:strong duality} and Lemma \ref{lemma:dual boundedness} together imply that, instead of directly solving the non-convex primal problem (\ref{prob:maxentrl}), one can seek to solve the convex dual problem
\begin{equation}\label{prob:dual problem}
\min_\Lamb  \ {D}(\Lamb)\quad  s.t. \ \Lamb \in \Lambda.
\end{equation}
However, there are two open problems that need to be addressed. The first one is that although algorithms in convex optimization can be used to solve the dual problem, it is not clear how fast they will converge without discovering key properties of the dual function. The second problem is that 
optimizing the dual function gives a dual optimality bound, while our goal is find a primal solution and analyze the primal optimality gap together with the constraint violation. 
In the following sections, we will show that the entropy-regularized CMDP has special structures that can be leveraged to address the above issues.

\subsection{Dual Smoothness} \label{sec:dual smoothness}
In optimization, smoothness plays an important role in establishing the convergence rate of an algorithm. 
Recall that a function $f:X \rightarrow \mbb{R}$ is $\ell$-smooth if
\begin{equation}
\left\|\nabla f(x_1)-\nabla f\left(x_2\right)\right\|_{2} \leq \ell\left\|x_1-x_2\right\|_{2},
\end{equation}
for all $x_1, x_2 \in X$.
In constrained optimization, however, smoothness is not always guaranteed, even when the primal problem is convex \citep{necoara2019complexity}.
In addition, while the subgradient of the dual function exists in general, the dual function is not always differentiable due to the non-uniqueness of Lagrangian multipliers. 

By leveraging the entropy regularization, we will show that the dual function $D(\Lamb)$ in CMDPs is both differentiable and $\ell$-smooth for some constant $\ell >0$, under the following assumption on the discounted state visitation distribution.
\begin{assumption}\label{assump:initial distribution}
The discounted state visitation distribution $d^\pi_\rho$ is uniformly bounded away from 0 for all $\pi\in \Pi$, i.e. there exists $d >0$, such that $d^\pi_\rho(s) \geq d$, $\forall s\in \mc{S}$, $\pi\in \Pi$.
\end{assumption}
Assumption \ref{assump:initial distribution} ensures that the MDP \textit{sufficiently explores} the state space.
Since $d_\rho^{\pi}(s) \geq (1-\gamma) \rho(s)$, it is satisfied when the initial distribution $\rho$ lies in the interior of the probability simplex $\Delta (\mc{S})$.
Similar assumptions are used in the prior literature  \citep{agarwal2021theory,mei2020global,mei2021leveraging}.

The following proposition is crucial for the development of our main result. 
\begin{proposition}\label{prop:quadratic lower bound}
For all policy $\pi$ and $\Lamb\geq 0$, it holds that
\begin{equation}
L(\pi_\Lamb, \Lamb)-L(\pi,\Lamb) \geq  \frac{\tau d }{2(1-\gamma)\ln 2} \left\|\pi - \pi_\Lamb\right\|_2^2.
\end{equation}
\end{proposition}
Under Assumption \ref{assump:initial distribution}, Proposition \ref{prop:quadratic lower bound} implies that the Lagrangian function $L(\pi,\Lamb)$ has a negative curvature at $\pi_\Lamb$ in all directions.
The proof relies on the soft sub-optimality lemma (cf. Lemma \ref{lemma:soft sub-optimality}) and a lower bound on the KL divergence \citep{cover1999elements}.

With the quadratic lower bound given by Proposition \ref{prop:quadratic lower bound}, we derive the following result.
\begin{proposition}\label{prop:dual smoothness}
Under Assumption \ref{assump:initial distribution},
 the dual function $D(\Lamb)$ satisfies the following properties:
\begin{enumerate}
    \item  $D(\Lamb)$ is differentiable and
\begin{align}
\nonumber\nabla D(\Lamb) =& U_\mb{g}^{\pi_\Lamb}(\Lamb) - \mb{b}\\
=& \left(U_{g_1}^{\pi_\Lamb}(\Lamb) - b_1, \dots, U_{g_n}^{\pi_\Lamb}(\Lamb) - b_n\right).
\end{align}
\item 
$D(\Lamb)$ is $\ell$-smooth on $\Lambda$, where 
\begin{equation}\label{eq:smoothness factor}
\ell = \frac{2\times \ln 2 \times \left(n|\mathcal{A}|+(1-\gamma)^2\sqrt{n|\mathcal{A}|}\right)}{\tau (1-\gamma)^3d}.
\end{equation}
\end{enumerate}
\end{proposition}
Proposition \ref{prop:dual smoothness} asserts that the dual function $D(\Lamb)$ is not only differentiable but also smooth on $\Lambda$.
This is desirable since, along with the convexity, it establishes an improved convergence rate compared with the slow convergence rate of sub-gradient methods.
We provide a short proof sketch for Proposition \ref{prop:dual smoothness} below:
\begin{enumerate}
    \item As subgradients of the dual function always exist for continuous problems, the differentiability follows from the uniqueness of the Lagrangian maximizer $\pi_\Lamb$ for every $\Lamb \in \Lambda$ \citep{floudas1995nonlinear}\footnote{Although more than enough, Proposition \ref{prop:quadratic lower bound} under Assumption \ref{assump:initial distribution} provides an intuitive way to think about the uniqueness of $\pi_\Lamb$.}.

\item The smoothness of $D(\Lamb)$ is the joint result of the Lipschitz continuity of $U_\mb{g}^\pi(\rho)$ with respect to $\pi$ (cf. Lemma \ref{lemma:lipschitz value function}) and the Lipschitz continuity of $\pi_\Lamb$ with respect to $\Lamb$, i.e. $\left\|\pi_{\Lamb_1}-\pi_{\Lamb_2}\right\|_2\leq \ell_\Lambda \|\Lamb_1-\Lamb_2\|_2$ for some $\ell_\Lambda > 0$.
To prove the latter, the main idea is to use the quadratic lower bound given by Proposition \ref{prop:quadratic lower bound} to conclude that $\pi_\Lamb$ is a second-order strict local maximum.
After that, we apply a standard result from perturbation analysis which states that $\pi_\Lamb$ is Lipschitz stable at $\Lamb$ \citep{bonnans2013perturbation}.
\end{enumerate}

\subsection{Optimality Gap And Constraint Violation}
Given a candidate solution $\pi$ to the CMDP problem in (\ref{prob:maxentrl}), our primary measures of the quality of the solution  $\pi$ are the primal optimality gap 
$\left |V_\tau^{\pi}(\rho)- V_\tau^\star\right|$, and the constraint violation $\max_{i\in[n]} \left[b_i - U_{g_i}^{\pi}(\rho) \right]_+$, where $[x]_+ := \max\{x,0\}$.
However, dual-descent based methods could only guarantee a convergence bound in terms of the dual optimality gap $D(\Lamb) - D^\star_\tau$.
In general, there is no guarantee that an $\varepsilon$-optimal dual solution $\Lamb$, namely $D(\Lamb) - D^\star_\tau\leq \varepsilon$,
would imply an $\mathcal{O}\left(\varepsilon^k\right)$ bound either on the primal optimality gap or on the constraint violation for the associated primal solution $\pi_\Lamb$ defined in \eqref{eq:pi_lambda evaluation}, for some $k\in (0,1]$.

However, in light of the entropy regularization, it is possible to show that an $\varepsilon$ error bound for dual functions would yield an $O\left(\sqrt{\varepsilon}\right)$ error bound for
the primal optimality gap and the constraint violation.
We summarize the results in the following proposition. 
\begin{proposition}\label{prop:dual to primal}
Suppose that Assumptions \ref{assump:slater}, \ref{assump:strong duality}, and \ref{assump:initial distribution} hold. If $\Lamb\geq 0$ is an $\varepsilon$-optimal multiplier, i.e. $D(\Lamb) -D^\star_\tau \leq \varepsilon$, then there exist $C_1,C_2>0$ such that the associated Lagrangian maximizer $\pi_\Lamb$ satisfies
\begin{subequations}
\begin{align}
\left\|\pi_\Lamb - \pi_\tau^\star\right\|_2 &\leq C_1 \sqrt{\varepsilon}, \label{eq:pi optimality_3}\\
\left |V_\tau^{\pi_\Lamb}(\rho)- V_\tau^\star\right| &\leq 2\varepsilon+\ell_cC_1C_2\sqrt{\varepsilon}, \label{eq:primal optimality_3}\\
\max_{i\in[n]} \left[b_i - U_{g_i}^{\pi_\Lamb}(\rho) \right]_+ &\leq \ell_cC_1 \sqrt{\varepsilon}, \label{eq:constraint violation_3} 
\end{align}
\end{subequations}
where 
$\ell_c$ is the Lipschitz constant defined in (\ref{eq:lipschitz value function}).
\end{proposition}
The values of the problem-dependent constants $C_1$ and $C_2$ can be found in Appendix \ref{app:1}.

In a nutshell, Proposition \ref{prop:dual to primal} enables the conversion of the dual optimality bound to primal metrics of interests, at the cost of enlarging the sub-optimality by a square root.
This is a non-trivial result and it does not hold in a general setting without the entropy regularization.
The proof of (\ref{eq:pi optimality_3}) relies on the quadratic lower bound given in Proposition \ref{prop:quadratic lower bound}.
Then, using the Lipschitz continuity of $U_\mb{g}^\pi(\rho)$ with respect to $\pi$ (cf. Lemma \ref{lemma:lipschitz value function}), we can derive the bound (\ref{eq:constraint violation_3}) on the constraint violation.
Finally, (\ref{eq:primal optimality_3}) can be obtained with some primal-dual properties.

 \section{FIRST-ORDER DUAL-DESCENT ALGORITHM}\label{sec:alg}
As shown in Section \ref{sec:property}, the dual function $D(\Lamb)$ for entropy-regularized CMDPs enjoys desirable properties, including the differentiability, the smoothness and the decomposition of the dual optimality gap.
These favorable properties of entropy-regularized CMDPs motivate us to use a dual-descent approach to solve the dual problem (\ref{prob:dual problem}). In particular, we choose {first-order methods, e.g. gradient projection method or Frank-Wolfe algorithm}, while using the Natural Policy Gradient (NPG) algorithm as a subroutine for evaluating $D(\Lamb)$ as well as $\nabla D(\Lamb)$.
To streamline the presentation, we mainly focus on the gradient projection method with the Nesterov acceleration as an example in the remainder of the paper.
We begin with a brief introduction about the NPG algorithm.
\subsection{Preliminary Tools}\label{subsec:preliminary tools}
\paragraph{NPG Algorithm With Entropy Regularization}
To optimize an unconstrained value function with respect to the policy, one commonly used first-order method is the 
{Natural Policy Gradient} algorithm \citep{kakade2001natural}, which deploys a pre-conditioned gradient update and regularizes the descent direction by the Fisher-information matrix $\mathcal{F}_{\rho}^{\theta}$ (cf. Appendix \ref{subsec:entropy-regularized npg_app}):
\begin{equation}\label{eq:NPG}
\theta \leftarrow \theta+\eta\left(\mathcal{F}_{\rho}^{\theta}\right)^{\dagger} \nabla_{\theta} V^{\pi_{\theta}}(\rho),
\end{equation}
where $\eta$ is the step-size and $(A)^{\dagger}$ denotes the Moore–Penrose inverse
of a matrix $A$.

In the entropy regularized setting, the update scheme is obtained by replacing $\nabla_{\theta} V^{\pi_{\theta}}(\rho)$ in (\ref{eq:NPG}) with $\nabla_{\theta} V_\tau^{\pi_{\theta}}(\rho)$.
Under the soft-max parameterization, the associated policy update has a fairly direct form, which is surprisingly independent from the initial distribution $\rho$:
\begin{equation}\label{eq:policy update}
\pi^{(t+1)}(a \vert s)\propto\left(\pi^{(t)}(a \vert s)\right)^{1-\frac{\eta \tau}{1-\gamma}} \exp \left(\frac{\eta Q_{\tau}^{\pi^{(t)}}(s, a)}{1-\gamma}\right),
\end{equation} 
where we use the shorthand $\pi^{(t)}$ for the soft-max parameterized policy with respect to $\theta^{(t)}$, and $Q^\pi_\tau$ is the soft Q-function defined in (\ref{def:soft Q function}).
The right-hand side of (\ref{eq:policy update}) can be normalized by multiplying a factor $Z^{(t)}(s)$, defined as 
\begin{equation}\label{eq:Zt}
Z^{(t)}(s) := \sum_{a\in \mathcal{A}} \left(\pi^{(t)}(a \vert s)\right)^{1-\frac{\eta \tau}{1-\gamma}} \exp \left(\frac{\eta Q_{\tau}^{\pi^{(t)}}(s, a)}{1-\gamma}\right),
\end{equation}
to make $\pi^{(t+1)}$ be a valid distribution.

\citet{cen2021fast} proved the global linear convergence of the entropy-regularized NPG method with a constant step-size.
In particular, the error bound $\left\|\log \pi^{\star}_\tau-\log \pi^{(t)}\right\|_{\infty} \leq \varepsilon$ can be achieved in 
\begin{equation}\label{eq:npg complexity}
\frac{1}{1-\gamma} \log \left(\frac{2\left\|Q_\tau^{\star}-Q_{\tau}^{\pi^{(0)}}\right\|_{\infty}}{\varepsilon\tau}\right),
\end{equation}
iterations with the step-size $\eta = (1-\gamma)/\tau$, where $\pi^\star_\tau$ is the optimal policy and $Q_\tau^\star (s,a):= Q^{\pi^\star_\tau}_\tau(s,a)$ is the associated optimal Q-function. 
Furthermore, they proved that the convergence rate becomes quadratic around the optimum.
{We refer the reader to Appendix \ref{subsec:entropy-regularized npg_app} for more details.}

\paragraph{Accelerated Gradient Projection Method with Inexact Gradient}
The Nesterov acceleration is a momentum-based approach that can be used to modify a gradient descent-type method to improve its convergence \citep{nesterov1983method, nesterov2013gradient}. Consider the optimization problem 
\begin{equation}\label{prob:constrained opt}
\begin{aligned}
\min_x & \ f(x) \quad s.t. & \ x \in X
\end{aligned}
\end{equation}
where $f(x)$ is convex and differentiable, and $X$ is a convex set.
The accelerated gradient projection method takes the update rule
\begin{equation}\label{eq:accelerated gradient projection}
\left\{\begin{aligned}
x^{(k+1)} &= \mathcal{P}_X\left(y^{(k)}-\alpha^k\nabla f(y^{(k)})\right)\\
y^{(k)} &= x^{(k)}+\beta_k\left(x^{(k)}-x^{(k-1)}\right)
\end{aligned}\right.
,\ k=0,1\dots
\end{equation}
where  $\mathcal{P}_X$ denotes the projection onto the set $X$, defined as $\mathcal{P}_X(y) := \operatorname{arg}\min_{x\in X} \|x-y\|_2$, and $\left\{\beta_k\right\}$ is chosen in a particular way to accelerate the convergence.
The iteration (\ref{eq:accelerated gradient projection}) first computes an extrapolation point $y^{(k)}$ and then performs the gradient projection update on $y^{(k)}$ to find the next point $x^{(k+1)}$.
It coincides with the standard gradient projection method when $\beta_k = 0$.
For a convex and smooth function $f$, 
the accelerated gradient projection method (\ref{eq:accelerated gradient projection}) achieves an error bound of $\mathcal{O}\left({1}/{T^2}\right)$ in $T$ iterations \citep{nesterov2013gradient}.

When the gradient evaluation is inexact with a bounded error $\delta$, i.e. we have access to some function $h:X \rightarrow \mbb{R}^n$ such that $\left\|\nabla f(x) - h(x) \right\|_2\leq \delta$ for all $x\in X$, \citet{schmidt2011convergence} proved that the accelerated gradient projection method still works with the slightly different error bound $\mathcal{O}\left({1}/{T^2} + T^2\delta^2+\delta \right)$.
Despite the $\mathcal{O}\left({1}/{T^2}\right)$ shrinking term, there is an accumulated error incurred by the inexact gradient. 
We refer the reader to Proposition \ref{prop:gradient projection_app} in Appendix \ref{app:2} for a formal statement.

\subsection{Accelerated Gradient Projection Method With NPG Subroutine}
Before presenting our method, we first note that the feasible region $\Lambda$, as defined in (\ref{eq:Lambda definition}), makes the dual problem (\ref{prob:dual problem}) amenable to many constrained optimization methods.
Especially, the projection operator $\mathcal{P}_\Lambda (\cdot)$ maps a point $\Lamb$ coordinate-wisely onto $\Lambda$ such that 
\begin{equation}\label{eq:lambda projection}
\left(\mathcal{P}_\Lambda(\Lamb)\right)_i = \operatorname{Median}\left\{0,\frac{V_\tau^\star-V_{\tau}^{\overline{\pi}}(\rho)}{\xi_i},\lambda_i\right\}
\end{equation}
where $\operatorname{Median}\{\cdot,\cdot,\cdot\}$ returns the median of the input numbers.

The proposed method works in two loops.
In the outer loop, we perform the accelerated gradient projection method on the dual function $D(\Lamb)$, whereas we use the natural policy gradient method in the inner loop to evaluate $D(\Lamb)$ by maximizing the Lagrangian $L(\pi,\Lamb)$ with respect to $\pi$. 
We summarize the details of our method in Algorithm \ref{alg:dualpg}.

Specifically, in line \ref{alg:extrapolation line}, we compute the extrapolation point $\muu^{(t)}$.
Then, in line \ref{alg:npg line}, we estimate the corresponding Lagrangian maximizer $\pi_{\muu^{(t)}}$, defined in \ref{eq:pi_lambda evaluation}, using the natural policy gradient subroutine, which is displayed in Algorithm \ref{alg:npg}.
With the estimated policy $\widetilde \pi_{\muu^{(t)}}$, we evaluate the gradient $\nabla D\left(\muu^{(t)}\right)$ by substituting the policy into the utility function in line \ref{alg:gradient evaluation line}.
In line \ref{alg:gradient projection line}, we perform the gradient projection update at $\muu^{(t)}$ using the estimated gradient $\widetilde \nabla D\left(\muu^{(t)}\right)$.
We remark that, as $V^\star_\tau$ is generally unknown, the projection $\mc{P}_\Lambda$ may not be precisely done in practical. 
Alternatively, one can perform the projection onto $\widetilde\Lambda := \left\{\Lamb \mid 0\leq \lambda_i \leq (2+2\tau \log\mc{A}) /((1-\gamma) \xi_i),\text{ for all } i \in [n]\right\}$.
Since the difference $V_{\tau}^{\star}-V_{\tau}^{\bar{\pi}}(\rho)$ is upper bounded by $(2+2\tau \log\mc{A}) /(1-\gamma)$, it holds that $\Lamb^\star \in \Lambda \subseteq \widetilde \Lambda$.
This reduction would not influence the order of convergence. 
Finally, upon the termination of the outer loop, we recover the primal variable (policy) from the dual variable by running the NPG subroutine for $N_3$ iterations in line \ref{alg:final recover}.

\begin{algorithm}[tb]
   \caption{Accelerated Gradient Projection Method with NPG Subroutine\label{alg:dualpg}}
\begin{algorithmic}[1]
   \STATE {\bfseries Input:} Initialization $\Lamb^{(-1)}, \Lamb^{(0)}$, $\widetilde\pi_{\muu^{(-1)}}$; step-size $\{\alpha_t\}_{t\geq 0}$, $\eta$; extrapolation weight $\{\beta_t\}_{t\geq 0}$; maximum number of iterations $N_1$, $N_2$, $N_3$.
  \FOR{$t = 0,1,2,\dots, N_1-1$}
  \STATE Compute the extrapolation point: \label{alg:extrapolation line}
  $
  \muu^{(t)} = \Lamb^{(t)}+\beta_t\left(\Lamb^{(t)}-\Lamb^{(t-1)}\right)
  $.
  \STATE Estimate the optimal policy $\pi_{\muu^{(t)}}$ for problem (\ref{eq:pi_lambda evaluation}) through the natural policy gradient subroutine: \label{alg:npg line}
$
	\widetilde \pi_{\muu^{(t)}} \leftarrow \text{NPG}_{Sub}\left(\muu^{(t)}, \widetilde\pi_{\muu^{(t-1)}}, \eta, N_2 \right)
$.
    \STATE Compute the approximate gradient at $\muu^{(t)}$: \label{alg:gradient evaluation line}
    $
	\widetilde \nabla D\left(\muu^{(t)}\right) := U_\mb{g}^{\widetilde \pi_{\muu^{(t)}}}(\rho) -\mb{b} $.
	\STATE Take a gradient projection step at $\mu^{(t)}$: $\Lamb^{(t+1)} \leftarrow \mathcal{P}_{\Lambda}\left(\muu^{(t)} -\alpha_t \widetilde \nabla D\left(\muu^{(t)}\right)\right)$, as defined by (\ref{eq:lambda projection}). \label{alg:gradient projection line}
  \ENDFOR
  \STATE Recover the policy from the dual variable: $\widetilde\pi_{\Lamb^{(N_1)}}\leftarrow \text{NPG}_{Sub}\left(\Lamb^{(N_1)}, \widetilde\pi_{\muu^{(N_1-1)}}, \eta, N_3 \right)$. \label{alg:final recover}
\end{algorithmic}
\end{algorithm}

\begin{algorithm}[tb]
   \caption{Natural Policy Gradient Subroutine (NPG$_{Sub}$)\label{alg:npg}}
\begin{algorithmic}[1]
   \STATE {\bfseries Input:} Multiplier $\Lamb$; initial policy $\pi^{(0)}$; step-size $\eta$; maximum number of iterations $N$.
  \FOR{$t = 0,1,2,\dots, N-1$}
  \STATE Compute the soft Q-function associated with the Lagrangian: $Q_\Lamb^{\pi^{(t)}}$ of policy $\pi^{(t)}$. 
  \STATE Update the policy with $Q_\Lamb^{\pi^{(t)}}$ through (\ref{eq:policy update}).
  \ENDFOR
\end{algorithmic}
\end{algorithm}

 \section{CONVERGENCE ANALYSIS}\label{sec:convergence}
In this section, we analyze the convergence of Algorithm \ref{alg:dualpg}.
The complete proofs of results in this section are postponed to Appendix \ref{subsec:convergence proof_app}.

A simple insight is that the NPG subroutine in Algorithm \ref{alg:dualpg} computes the optimal
 policy at a linear rate in the inner loop and the accelerated gradient projection algorithm converges in $\mathcal{O}\left(1/\sqrt{\varepsilon}\right)$ rate in the outer loop, leading to the overall convergence rate of $\widetilde{\mathcal{O}}\left(1/\sqrt{\varepsilon}\right)$.
Then, we obtain the desired $\widetilde{\mathcal{O}}(1/\varepsilon)$ rate in terms of the primal optimality gap and constraint violation by applying Proposition \ref{prop:dual to primal}.

The above high-level technique requires subtle technicalities to be addressed here.  
We begin with the following proposition, which evaluates the accuracy of the gradient estimator defined in line \ref{alg:gradient evaluation line} of Algorithm \ref{alg:dualpg}.
\begin{proposition}\label{prop:gradient evaluation}
Suppose that $\pi$ is an approximate solution to (\ref{eq:pi_lambda evaluation}) such that $\|\log \pi - \log \pi_\Lamb\|_\infty \leq \varepsilon$. 
The gradient estimator defined as $\widetilde \nabla D\left(\Lamb\right) :=U_\mb{g}^{\pi}(\rho) -\mb{b} =  \left(U_{g_1}^{\pi}(\rho) -b_1,\dots,U_{g_n}^{\pi}(\rho) -b_n  \right)$ satisfies the inequality $\left\|\widetilde \nabla D\left(\lambda\right) - \nabla D\left(\lambda\right)\right\|_2  \leq {\sqrt{n}|\mathcal{A}|\varepsilon}/{(1-\gamma)^2}$.
\end{proposition}
This result can be deduced from the inequality $\|\pi - \pi_\Lamb\|_\infty \leq \|\log \pi - \log \pi_\Lamb\|_\infty \leq \varepsilon$ and the performance difference lemma of an unregularized value function (Lemma \ref{lemma:performance difference}).
Recall that $L(\pi,\Lamb) = V_\Lamb^\pi(\rho) - \Lamb^\top\mb{b}$, where $V_\Lamb^\pi(\rho)= V_{\tau}^\pi(\rho) + \Lamb^\top U_\mb{g}^\pi (\rho)$ is the value function with the reward $r_{\Lamb} (s,a) :=r(s,a) + \Lamb^\top\mb{g}(s,a)$ (cf. Section \ref{sec:formulation})).
Therefore, Proposition \ref{prop:gradient evaluation} together with (\ref{eq:npg complexity}), implies that running Algorithm \ref{alg:npg} with the step-size $\eta = (1-\gamma)/\tau$ for 
\begin{equation}\label{eq:N_2}
\frac{1}{1-\gamma}\log \left( \frac{2\sqrt{n} \left|\mathcal{A}\right|\left\|Q_\Lamb^{\pi_\Lamb} -Q_\Lamb^{\pi^{(0)}}\right\|_\infty}{\delta (1-\gamma)^2\tau} \right)
\end{equation}
iterations can guarantee a $\delta$-accurate gradient estimation $\widetilde \nabla D\left(\Lamb\right)$, i.e.  $\left\|\widetilde \nabla D\left(\Lamb\right) - \nabla D\left(\Lamb\right)\right\|_2 \leq \delta$.

Below, we present our main convergence result of Algorithm \ref{alg:dualpg}.
\begin{theorem}\label{thm:convergence}
Suppose that Assumptions \ref{assump:slater}, \ref{assump:strong duality}, and \ref{assump:initial distribution} hold. For every $\varepsilon_1>0$, there exist some constants $C_1$ and $C_2 >0$ such that Algorithm \ref{alg:dualpg} with a random initialization and the parameters $\eta = (1-\gamma)/\tau$, $\alpha_k = {1}/{\ell}$, $\beta_k = {(k-1)}/{(k+2)}$, $N_1 = T$, $N_2= \mathcal{O}\left(\log T\right)$ and $N_3=  \mathcal{O}\left(\log 1/\varepsilon_1\right)$ returns a solution pair $(\pi,\Lamb)$ such that
\begin{subequations}
\begin{align}
&D(\Lamb) - D^\star_\tau \leq \varepsilon_0,\label{eq:thm results_0}\\
&\left\|\pi - \pi^\star_\tau\right\|_2 \leq C_1\sqrt{\varepsilon_0} + \varepsilon_1,\label{eq:thm results_1}\\
&\left|V_\tau^{\pi}(\rho) - V^\star_\tau\right| \leq 2\varepsilon_0 + \ell_cC_1C_2\sqrt{\varepsilon_0} + \left(\ell_cC_2 + \frac{3\gamma}{2\tau\sqrt{n}}\right)\varepsilon_1,\label{eq:thm results_3}\\
&\max_{i\in [n]} \left[b_i - U_{g_i}^{\pi}(\rho) \right]_+ \leq \ell_c\left(C_1\sqrt{\varepsilon_0} +\varepsilon_1 \right),\label{eq:thm results_2}
\end{align}
\end{subequations}
where 
\begin{equation}
\varepsilon_0 = \frac{2\ell}{(T+1)^2}\left(\left\|\Lamb^{(0)}-\Lamb^{\star}\right\|_2+1\right)^{2},
\end{equation}
and where
$\ell$ is the smoothness factor defined in (\ref{eq:smoothness factor}) and $\ell_c$ is the Lipschitz constant defined in (\ref{eq:lipschitz value function}).
The total iteration complexity is $N_1\times N_2 +N_3 = \widetilde{\mathcal{O}}(T)$ with primal error bounds ${\mathcal{O}}\left({1}/{T}\right)$ given by (\ref{eq:thm results_1})-(\ref{eq:thm results_2}), and a dual error bound ${\mathcal{O}}\left({1}/{T^2}\right)$ given by (\ref{eq:thm results_0}).
\end{theorem}
The values of the parameters $N_2$, $N_3$, and problem-dependent constants $C_1$, $C_2$ can be found in Theorem \ref{thm:convergence_app}, Appendix \ref{subsec:convergence proof_app}, where we restate the theorem.

Theorem \ref{thm:convergence} shows that Algorithm \ref{alg:dualpg} achieves a global convergence with the rate $\widetilde{\mathcal{O}}\left(1/\varepsilon\right)$.
Specifically, it shows that with $\widetilde{\mathcal{O}}(T)$ number of iterations in total, Algorithm \ref{alg:dualpg} generates a solution with ${\mathcal{O}}(1/T)$ error bounds in terms of the policy (primal variable), primal optimality gap, and constraint violation, as well as ${\mathcal{O}}(1/T^2)$ error in terms of the dual optimality gap.

We briefly describe the intuition behind the proof of Theorem \ref{thm:convergence} below. 
\begin{enumerate}
    \item Firstly, the linear convergence of the natural policy gradient method (cf. Proposition \ref{prop:cen_app}) and Proposition \ref{prop:gradient evaluation} imply that running the NPG subroutine for $N_2 = \mathcal{O}(\log T)$ iterations in the inner loop guarantees a sufficiently accurate estimation of $\nabla D(\Lamb)$. 
    \item Then, we apply the convergence result by \citet{schmidt2011convergence} (refer to Section \ref{subsec:preliminary tools} and Appendix \ref{subsec:gradient projection}), which implies the dual optimality gap $\mathcal{O}\left({1}/{T^2} + T^2\delta^2+\delta \right)$, where $\delta$ is the gradient estimation error. Since the NPG subroutine converges linearly, we can suppress the constant in $N_2 = \mathcal{O}(\log T)$, and make $\delta$ sufficiently small such that the dual optimality gap equals $\mathcal{O}\left({1}/{T^2}\right)$.
    \item Let $\Lamb$ denote the dual variable returned by the for loop.
    Running the NPG subroutine in line \ref{alg:final recover} for additional $N_3 = \mathcal{O}\left(\log (1/\varepsilon_1)\right)$ iterations guarantees an $\varepsilon_1$-approximate solution $\pi$ for the Lagrangian maximizer $\pi_\Lamb$. Again, we can make $\varepsilon_1$ sufficiently small by suppressing the constant in $N_3$.
    \item Finally, by applying Proposition \ref{prop:dual to primal} and the triangular inequality, we prove (\ref{eq:thm results_1}), stating that $\pi$ is $\mathcal{O}(1/T)$-optimal. We bound the constraint violation (\ref{eq:thm results_2}) by using the Lipschitz continuity of $U_\mb{g}^\pi(\rho)$ with respect to $\pi$, and bound the primal optimality gap (\ref{eq:thm results_3}) by using some primal-dual properties\footnote{The techniques are similar to those in the proof of Proposition \ref{prop:dual to primal}.}.
\end{enumerate}
\begin{remark}[Quadratic Convergence of NPG]
Our analysis of Algorithm \ref{alg:dualpg} in this section is inspired by the global linear convergence of the entropy-regularized NPG method.
However, as \citet{cen2021fast} proved, the NPG method achieves a quadratic convergence around the optimum (cf. Proposition \ref{prop:cen_quadratic_app}).
Therefore, it may be possible to improve the hidden constants in $N_2 = \mathcal{O}(\log T)$ and $N_3 = \mathcal{O}\left(\log (1/\varepsilon_1)\right)$ under extra assumptions. 
\end{remark}


So far, we have only studied the entropy-regularized CMDP.
However, adding entropy induces bias to the optimal solution of the standard unregularized CMDP.
A standard way to deal with this mismatch issue is to choose the regularization parameter $\tau$ to be sufficiently small.
The following corollary shows that we can compute a near-optimal policy with the rate $\widetilde{\mathcal{O}}\left(1/\sqrt{T}\right)$ for both the optimality gap and the constraint violation for the standard CMDP.
\begin{corollary}\label{cor:original mdp solution}
Suppose that Assumptions \ref{assump:slater}, \ref{assump:strong duality}, and \ref{assump:initial distribution} hold. 
{Then, Algorithm \ref{alg:dualpg} with the choice $\tau = \mathcal{O}(\varepsilon)$}
computes a solution $\pi$ for the standard CMDP such that
\begin{subequations}
\begin{align}
\left|V^{\pi^\star}(\rho)-V^{\pi}(\rho)\right| & = \mathcal{O}(\varepsilon),\label{eq:cor_1}\\
\max_{i\in [n]} \left[b_i - U_{g_i}^{\pi}(\rho) \right]_+ &= \mathcal{O}(\varepsilon),\label{eq:cor_2}
\end{align}
\end{subequations}
in $\widetilde{\mathcal{O}}\left( 1/\varepsilon^2 \right)$ iterations, where $\pi^\star$ is an optimal policy to the standard CMDP.
\end{corollary}
The proof of Corollary \ref{cor:original mdp solution} relies on the following sandwich bound
\begin{equation}\label{eq:original and regularized bound}
V^{\pi_{\tau}^{\star}}(\rho) \leq V^{\pi_{\star}}(\rho) \leq V^{\pi_{\tau}^{\star}}(\rho)+{\frac{\tau}{1-\gamma} \log |\mathcal{A}|}.
\end{equation}


\section{CMDPS WITH A SINGLE CONSTRAINT}\label{sec:single constraint}
{Since we can convert the dual optimality bound to primal metrics of interests with little extra effort (cf. Proposition \ref{prop:dual to primal}), the overall complexity relies on how fast we can solve the dual problem (\ref{prob:dual problem}).
For the special case where $n=1$, the dual problem amounts to optimizing a convex function on an closed interval, which can be efficiently solved by the bisection method.
Due to space restrictions, we refer the reader to Appendix \ref{subsec:bisection_app} for more details about the algorithm (cf. Algorithm \ref{alg:bisection}).
We state the result in the following theorem.}
\begin{theorem}\label{thm:bisection convergence}
{Suppose that Assumptions \ref{assump:slater}, \ref{assump:strong duality}, and \ref{assump:initial distribution} hold. When $n=1$, for every $\varepsilon_0$, $\varepsilon_1>0$, Algorithm \ref{alg:bisection} returns a solution pair $(\pi,\lambda)$ satisfying (\ref{eq:thm results_0})-(\ref{eq:thm results_2}) in at most $\mathcal{O}\left(\log^2(1/\varepsilon_0) + \log(1/\varepsilon_1)\right)$ iterations.}
\end{theorem}
{By leveraging the linear convergence, we can derive a result analogous to Corollary \ref{cor:original mdp solution} for the standard CMDP, but with a linear rate.}
\begin{corollary}\label{cor:for bisection}
{Suppose that Assumptions \ref{assump:slater}, \ref{assump:strong duality}, and \ref{assump:initial distribution} hold. 
Then, Algorithm \ref{alg:bisection} with the choice $\tau = \mathcal{O}(\varepsilon)$
computes a solution $\pi$ for the standard CMDP satisfying (\ref{eq:cor_1})-(\ref{eq:cor_2}), in $\mathcal{O}\left(\log^2(1/\varepsilon)\right)$ iterations.}
\end{corollary}

{We refer the reader to Appendix \ref{subsec:bisection_app} for the formal statements of Theorem \ref{thm:bisection convergence} and Corollary \ref{cor:for bisection} as well as their proofs.}


 \section{CONCLUSION}
In this paper, we showed that entropy regularization induces desirable properties to CMDPs from an optimization perspective. In particular, the Lagrangian dual function of CMDPs is smooth and an $\mathcal{O}\left({\varepsilon}\right)$ error bound for the dual optimality gap yields an $\mathcal{O}\left(\sqrt{\varepsilon}\right)$ error bound for the primal optimality gap and the constraint violation. In addition, we proposed a novel accelerated dual-descent algorithm and proved that it achieves a global convergence with the rate $\widetilde{\mathcal{O}}(1/\varepsilon)$ for both the optimality gap and the constraint violation.
It remains as an open question whether similar improved convergence results for the entropy-regularized CMDPs can be obtained with a sample-based policy gradient.

\section*{ACKNOWLEDGEMENTS}
This work was supported by grants from AFOSR, ARO, ONR, NSF and C3.ai Digital Transformation Institute.

\bibliographystyle{apalike}
\bibliography{ref}

\begin{thebibliography}{}

\bibitem[Achiam et~al., 2017]{achiam2017constrained}
Achiam, J., Held, D., Tamar, A., and Abbeel, P. (2017).
\newblock Constrained policy optimization.
\newblock In {\em International Conference on Machine Learning}, pages 22--31.
  PMLR.

\bibitem[Agarwal et~al., 2021]{agarwal2021theory}
Agarwal, A., Kakade, S.~M., Lee, J.~D., and Mahajan, G. (2021).
\newblock On the theory of policy gradient methods: Optimality, approximation,
  and distribution shift.
\newblock {\em Journal of Machine Learning Research}, 22(98):1--76.

\bibitem[Altman, 1999]{altman1999constrained}
Altman, E. (1999).
\newblock {\em Constrained Markov decision processes}, volume~7.
\newblock CRC Press.

\bibitem[Bhatnagar and Lakshmanan, 2012]{bhatnagar2012online}
Bhatnagar, S. and Lakshmanan, K. (2012).
\newblock An online actor--critic algorithm with function approximation for
  constrained markov decision processes.
\newblock {\em Journal of Optimization Theory and Applications},
  153(3):688--708.

\bibitem[Bonnans and Shapiro, 2013]{bonnans2013perturbation}
Bonnans, J.~F. and Shapiro, A. (2013).
\newblock {\em Perturbation analysis of optimization problems}.
\newblock Springer Science \& Business Media.

\bibitem[Borkar, 2005]{borkar2005actor}
Borkar, V.~S. (2005).
\newblock An actor-critic algorithm for constrained markov decision processes.
\newblock {\em Systems \& control letters}, 54(3):207--213.

\bibitem[Cen et~al., 2021]{cen2021fast}
Cen, S., Cheng, C., Chen, Y., Wei, Y., and Chi, Y. (2021).
\newblock Fast global convergence of natural policy gradient methods with
  entropy regularization.

\bibitem[Chen et~al., 2021]{chen2021primal}
Chen, Y., Dong, J., and Wang, Z. (2021).
\newblock A primal-dual approach to constrained markov decision processes.
\newblock {\em arXiv preprint arXiv:2101.10895}.

\bibitem[Chow et~al., 2017]{chow2017risk}
Chow, Y., Ghavamzadeh, M., Janson, L., and Pavone, M. (2017).
\newblock Risk-constrained reinforcement learning with percentile risk
  criteria.
\newblock {\em The Journal of Machine Learning Research}, 18(1):6070--6120.

\bibitem[Cover, 1999]{cover1999elements}
Cover, T.~M. (1999).
\newblock {\em Elements of information theory}.
\newblock John Wiley \& Sons.

\bibitem[Ding et~al., 2021]{ding2021provably}
Ding, D., Wei, X., Yang, Z., Wang, Z., and Jovanovic, M. (2021).
\newblock Provably efficient safe exploration via primal-dual policy
  optimization.
\newblock In {\em International Conference on Artificial Intelligence and
  Statistics}, pages 3304--3312. PMLR.

\bibitem[Ding et~al., 2020]{ding2020natural}
Ding, D., Zhang, K., Basar, T., and Jovanovic, M.~R. (2020).
\newblock Natural policy gradient primal-dual method for constrained markov
  decision processes.
\newblock In {\em Conference on Neural Information Processing Systems}.

\bibitem[Efroni et~al., 2020]{efroni2020exploration}
Efroni, Y., Mannor, S., and Pirotta, M. (2020).
\newblock Exploration-exploitation in constrained mdps.
\newblock {\em arXiv preprint arXiv:2003.02189}.

\bibitem[Fisac et~al., 2018]{fisac2018general}
Fisac, J.~F., Akametalu, A.~K., Zeilinger, M.~N., Kaynama, S., Gillula, J., and
  Tomlin, C.~J. (2018).
\newblock A general safety framework for learning-based control in uncertain
  robotic systems.
\newblock {\em IEEE Transactions on Automatic Control}, 64(7):2737--2752.

\bibitem[Floudas, 1995]{floudas1995nonlinear}
Floudas, C.~A. (1995).
\newblock {\em Nonlinear and mixed-integer optimization: fundamentals and
  applications}.
\newblock Oxford University Press.

\bibitem[Haarnoja et~al., 2017]{haarnoja2017reinforcement}
Haarnoja, T., Tang, H., Abbeel, P., and Levine, S. (2017).
\newblock Reinforcement learning with deep energy-based policies.
\newblock In {\em International Conference on Machine Learning}, pages
  1352--1361. PMLR.

\bibitem[Haarnoja et~al., 2018]{haarnoja2018soft}
Haarnoja, T., Zhou, A., Abbeel, P., and Levine, S. (2018).
\newblock Soft actor-critic: Off-policy maximum entropy deep reinforcement
  learning with a stochastic actor.
\newblock In {\em International conference on machine learning}, pages
  1861--1870. PMLR.

\bibitem[Kakade, 2001]{kakade2001natural}
Kakade, S.~M. (2001).
\newblock A natural policy gradient.
\newblock {\em Advances in neural information processing systems}, 14.

\bibitem[Li et~al., 2021]{li2021faster}
Li, T., Guan, Z., Zou, S., Xu, T., Liang, Y., and Lan, G. (2021).
\newblock Faster algorithm and sharper analysis for constrained markov decision
  process.
\newblock {\em arXiv preprint arXiv:2110.10351}.

\bibitem[Mei et~al., 2021]{mei2021leveraging}
Mei, J., Gao, Y., Dai, B., Szepesvari, C., and Schuurmans, D. (2021).
\newblock Leveraging non-uniformity in first-order non-convex optimization.
\newblock {\em arXiv preprint arXiv:2105.06072}.

\bibitem[Mei et~al., 2020]{mei2020global}
Mei, J., Xiao, C., Szepesvari, C., and Schuurmans, D. (2020).
\newblock On the global convergence rates of softmax policy gradient methods.
\newblock In {\em International Conference on Machine Learning}, pages
  6820--6829. PMLR.

\bibitem[Mnih et~al., 2016]{mnih2016asynchronous}
Mnih, V., Badia, A.~P., Mirza, M., Graves, A., Lillicrap, T., Harley, T.,
  Silver, D., and Kavukcuoglu, K. (2016).
\newblock Asynchronous methods for deep reinforcement learning.
\newblock In {\em International conference on machine learning}, pages
  1928--1937. PMLR.

\bibitem[Nachum et~al., 2017]{nachum2017bridging}
Nachum, O., Norouzi, M., Xu, K., and Schuurmans, D. (2017).
\newblock Bridging the gap between value and policy based reinforcement
  learning.
\newblock {\em arXiv preprint arXiv:1702.08892}.

\bibitem[Necoara et~al., 2019]{necoara2019complexity}
Necoara, I., Patrascu, A., and Glineur, F. (2019).
\newblock Complexity of first-order inexact lagrangian and penalty methods for
  conic convex programming.
\newblock {\em Optimization Methods and Software}, 34(2):305--335.

\bibitem[Nesterov, 2013]{nesterov2013gradient}
Nesterov, Y. (2013).
\newblock Gradient methods for minimizing composite functions.
\newblock {\em Mathematical Programming}, 140(1):125--161.

\bibitem[Nesterov, 1983]{nesterov1983method}
Nesterov, Y.~E. (1983).
\newblock A method for solving the convex programming problem with convergence
  rate o (1/k\^{} 2).
\newblock In {\em Dokl. akad. nauk Sssr}, volume 269, pages 543--547.

\bibitem[Paternain et~al., 2019]{paternain2019safe}
Paternain, S., Calvo-Fullana, M., Chamon, L.~F., and Ribeiro, A. (2019).
\newblock Safe policies for reinforcement learning via primal-dual methods.
\newblock {\em arXiv preprint arXiv:1911.09101}.

\bibitem[Schmidt et~al., 2011]{schmidt2011convergence}
Schmidt, M., Roux, N.~L., and Bach, F. (2011).
\newblock Convergence rates of inexact proximal-gradient methods for convex
  optimization.
\newblock {\em arXiv preprint arXiv:1109.2415}.

\bibitem[Uchibe and Doya, 2007]{uchibe2007constrained}
Uchibe, E. and Doya, K. (2007).
\newblock Constrained reinforcement learning from intrinsic and extrinsic
  rewards.
\newblock In {\em 2007 IEEE 6th International Conference on Development and
  Learning}, pages 163--168. IEEE.

\bibitem[Williams and Peng, 1991]{williams1991function}
Williams, R.~J. and Peng, J. (1991).
\newblock Function optimization using connectionist reinforcement learning
  algorithms.
\newblock {\em Connection Science}, 3(3):241--268.

\bibitem[Xu et~al., 2021]{xu2021crpo}
Xu, T., Liang, Y., and Lan, G. (2021).
\newblock Crpo: A new approach for safe reinforcement learning with convergence
  guarantee.
\newblock In {\em International Conference on Machine Learning}, pages
  11480--11491. PMLR.

\bibitem[Zang et~al., 2020]{zang2020teac}
Zang, H., Li, X., Zhang, L., Zhao, P., and Wang, M. (2020).
\newblock Teac: Intergrating trust region and max entropy actor critic for
  continuous control.

\bibitem[Zhang et~al., 2019]{zhang2019non}
Zhang, X., Zhang, K., Miehling, E., and Basar, T. (2019).
\newblock Non-cooperative inverse reinforcement learning.
\newblock {\em Advances in Neural Information Processing Systems}, 32.

\bibitem[Ziebart, 2010]{ziebart2010modeling}
Ziebart, B.~D. (2010).
\newblock {\em Modeling purposeful adaptive behavior with the principle of
  maximum causal entropy}.
\newblock Carnegie Mellon University.

\end{thebibliography}

\clearpage
\appendix
\thispagestyle{empty}
\onecolumn
\makesupplementtitle

\section{Proofs of Results in Sections \ref{sec:formulation} and \ref{sec:property}}\label{app:1}


\begin{lemma}[Restatement of Lemma \ref{lemma:lipschitz value function}]\label{lemma:lipschitz value function_app}
For an unregularized value function $V^\pi(\rho)$ with the reward function $r(s,a)\in [0,1]$, it holds that
\begin{equation}\label{eq:lipschitz value function_app}
\left|V^{\pi_1}(\rho) -V^{\pi_2}(\rho) \right| \leq \ell_c \left\|\pi_1 - \pi_2\right\|_2,
\end{equation}
for arbitrary policies $\pi_1$ and $\pi_2$, where $\ell_c = {\sqrt{|\mathcal{A}|}}/{(1-\gamma)^2}$.
\end{lemma}

\begin{proof}
It follows from the policy gradient for the direct parameterization (Lemma \ref{lemma:policy gradient_direct}) that
\begin{equation}
\frac{\partial V^{\pi}(\rho)}{\partial \pi(a \vert s)}=\frac{1}{1-\gamma} d_{\rho}^{\pi}(s) Q^{\pi}(s, a).
\end{equation}
Thus, we can bound $\nabla_\pi V^\pi (\rho)$ as
\begin{equation}\label{eq:lipschitz value function computation}
\begin{aligned}
\left\|\nabla_\pi V^\pi (\rho)\right\|_2 
&= \frac{1}{1-\gamma}\sqrt{\sum_{s\in \mathcal{S}, a\in \mathcal{A}} \left(d_{\rho}^{\pi}(s)Q^\pi (s,a)\right)^2}\\
&\leq \frac{\max_{s\in \mathcal{S}, a\in \mathcal{A}} Q^\pi (s,a)}{1-\gamma}\sqrt{\sum_{s\in \mathcal{S}, a\in \mathcal{A}} \left(d_{\rho}^{\pi}(s)\right)^2}\\
&\overset{(i)}\leq \frac{\sqrt{|\mathcal{A}|}}{(1-\gamma)^2} \left\|d^\pi_\rho (\cdot)\right\|_2\\
&\overset{(ii)}\leq \frac{\sqrt{|\mathcal{A}|}}{(1-\gamma)^2} =: \ell_c,
\end{aligned}
\end{equation}
where $(i)$ uses $Q^\pi (s,a) \leq {1}/{1-\gamma}$ and $(ii)$ is because of $\left\|d^\pi_\rho (\cdot)\right\|_2 \leq \left\|d^\pi_\rho (\cdot)\right\|_1 = 1$.
Then, (\ref{eq:lipschitz value function_app}) follows from
\begin{equation}
\left|V^{\pi_1}(\rho) -V^{\pi_2}(\rho) \right| \leq \sup_\pi \left\{\left\|\nabla_\pi V^\pi (\rho)\right\|_2  \right\} \left\|\pi_1 - \pi_2\right\|_2 \leq \ell_c \left\|\pi_1 - \pi_2\right\|_2.
\end{equation}
This completes the proof.
\end{proof}


\begin{lemma}[Restatement of Lemma \ref{lemma:dual boundedness}]\label{lemma:dual boundedness_app}
Under Assumptions \ref{assump:slater} and \ref{assump:strong duality}, it holds that 
\begin{equation}
0 \leq \lambda^{\star}_i \leq\frac{V_\tau^\star-V_{\tau}^{\overline{\pi}}(\rho)}{\xi_i},\quad \forall i \in [n].
\end{equation}
\end{lemma}
\begin{proof}
Let $C\in \mbb{R}$. For every $\Lamb\geq 0$ such that $D(\Lamb) \leq C$, one can write
\begin{equation}\label{eq:dual bound}
\begin{aligned}
C\geq D(\Lamb) &\overset{(i)}{\geq} V_\tau^{\overline \pi}(\rho) + \sum_{i=1}^n \lambda_i \left(U_{g_i}^{\overline \pi}(\rho)-b_i \right) \\
&\overset{(ii)}{\geq} V_\tau^{\overline \pi}(\rho) + \sum_{i=1}^{n} \lambda_i \xi_i,
\end{aligned}
\end{equation}
where $(i)$ follows from the definition of $D(\Lamb)$ and $(ii)$ is due to Assumption \ref{assump:slater}.
Since $\boldsymbol {\xi} >0$ and $\Lamb \geq 0$, (\ref{eq:dual bound}) gives rise to the bound $0 \leq \lambda_i \leq\left(C-V_{\tau}^{\overline{\pi}}(\rho)\right) / \xi_i$, for $i = 1,2, \dots, n$.
Now, by letting $C = V_\tau^\star$, it results from the strong duality that $\left\{\Lamb \geq 0 \mid D(\Lamb) \leq C\right\}$ becomes the set of optimal dual variables.
This completes the proof.
\end{proof}

\begin{proposition}[Restatement of Proposition \ref{prop:quadratic lower bound}]\label{prop:quadratic lower bound_app}
For all policy $\pi$ and $\Lamb\geq 0$, it holds that
\begin{equation}
L(\pi_\Lamb, \Lamb)-L(\pi,\Lamb) \geq \frac{\tau d }{2(1-\gamma)\ln 2}\left\|\pi - \pi_\Lamb\right\|_2^2.
\end{equation}
\end{proposition}
\begin{proof}
Recall that $L(\pi,\Lamb) = V_\Lamb^\pi(\rho) - \Lamb^\top\mb{b}$, where $V_\Lamb^\pi(\rho) = V_{\tau}^\pi(\rho) + \Lamb^\top U_\mb{g}^\pi (\rho)$ is the value function with the reward $r_{\Lamb} (s,a) :=r(s,a) + \Lamb^\top\mb{g}(s,a)$ (cf. Section \ref{sec:formulation}).
The soft sub-optimality gap (Lemma \ref{lemma:soft sub-optimality}) then gives
\begin{equation}\label{eq:lagrangian sub-optimality}
L(\pi_\Lamb, \Lamb)-L(\pi,\Lamb) =  V_\Lamb^{\pi_\Lamb}(\rho) -V_\Lamb^\pi(\rho) =  \frac{\tau }{1-\gamma}\sum_{s\in \mathcal{S}}d_\rho^\pi (s) D_\text{KL} \left[\pi(\cdot \vert s) \mid \pi_\Lamb (\cdot \vert s) \right],
\end{equation}
where $D_\text{KL} \left[P(\cdot) \mid Q(\cdot) \right] := \sum _{x} P(x)\left(\log P(x) -\log Q(x)\right)$ is the KL divergence between probability distributions $P(\cdot)$ and $Q(\cdot)$.
Now, we use a well-known bound relating the KL divergence to the vector 1-norm \citep{cover1999elements}:
\begin{equation}\label{eq:klinequality}
D_\text{KL}\left[P(\cdot) \mid Q(\cdot) \right] \geq \frac{1}{2 \ln 2} \|P(\cdot)-Q(\cdot)\|_1^2.
\end{equation}
Combine (\ref{eq:lagrangian sub-optimality}) and (\ref{eq:klinequality}) yields 
\begin{equation}
L(\pi_\Lamb, \Lamb)-L(\pi,\Lamb) \geq  \frac{\tau}{2(1-\gamma)\ln 2}\sum_{s\in \mathcal{S}}d_\rho^\pi (s)\|\pi(\cdot\vert s)-\pi_\Lamb(\cdot\vert s)\|_1^2 \overset{(i)}\geq  \frac{\tau d}{2(1-\gamma)\ln 2}\left\|\pi - \pi_\Lamb\right\|_2^2 ,
\end{equation}
where (i) is due to $\|\cdot\|_1 \geq \|\cdot\|_2$ and Assumption \ref{assump:initial distribution}.
This completes the proof.
\end{proof}

\begin{proposition}[Restatement of Proposition \ref{prop:dual smoothness}]\label{prop:dual smoothness_app}
Under Assumption \ref{assump:initial distribution},
 the dual function $D(\Lamb)$ satisfies the following properties:
\begin{enumerate}
    \item  $D(\Lamb)$ is differentiable and
\begin{equation}
\nabla D(\Lamb) = U_\mb{g}^{\pi_\Lamb}(\Lamb) - \mb{b}= \left(U_{g_1}^{\pi_\Lamb}(\Lamb) - b_1, \dots, U_{g_n}^{\pi_\Lamb}(\Lamb) - b_n\right).
\end{equation}
\item 
$D(\Lamb)$ is $\ell$-smooth on $\Lambda$, where 
\begin{equation}\label{eq:smoothness factor_app}
\ell = \frac{2\times \ln 2 \times \left(n|\mathcal{A}|+(1-\gamma)^2\sqrt{n|\mathcal{A}|}\right)}{\tau (1-\gamma)^3 d}.
\end{equation}
\end{enumerate}
\end{proposition}
The proof of Proposition \ref{prop:dual smoothness_app} relies on the following result by \citet{bonnans2013perturbation}.
\begin{proposition}[Lipschitz stability of parametric local maximizersz]\label{prop:parametric optimization}
Given a set $T\subset \mbb{R}^p$, consider a parametric optimization problem $P(t)$ with $t\in T$, stated as 
\begin{equation}
\max_{x\in F} \ f(x,t) \quad
s.t. \ F = \{x\in \mathbb{R}^n  \mid h_j(x)\leq 0,\ j=1,2,\dots,m\},
\end{equation}
where $f$, $h_j$ are twice continuously differentiable and $F \neq \emptyset$.
For every $\overline{t}\in T$, if $\overline x = x\left(\overline t\right)$ is a strict local maximizer of $P(\overline t)$ of order 2, i.e. $\nabla^2_{xx} f(\overline x, \overline t) \succeq w_1I_n$ for some $w_1>0$, 
then there exist $\varepsilon$, $\delta$, $L >0$ such that for all $t\in B_\varepsilon(\overline t) := \{t \mid \|t-\overline t \|_2 < \varepsilon\}$, there exists at least one local maximizer $x(t) \in B_\delta (\overline x)$ of $P(t)$ and for each such local maximizer we have 
\begin{equation}
\|x(t) - \overline x\|_2 \leq L \|t-\overline t\|_2.
\end{equation}
Especially, taking $L = {w_2}/{w_1}$ fulfills the requirement, with 
\begin{equation}
w_2=\max _{z \in \mathrm{cl}\left(B_{\delta}(\overline{x})\right)}\left[\left\|\nabla_{x t}^{2} f(z, \overline{t})\right\|_F+1\right].
\end{equation}
\end{proposition}

\textit{Proof of Proposition \ref{prop:dual smoothness_app}.} 
We first prove the differentiability of $D(\Lamb)$. 
For a fixed $\Lamb$, solving for $\pi_\Lamb= \operatorname{arg}\max_{\pi\in \Pi} L(\pi,\Lamb)$ is equivalent to solving an unconstrained MDP with  entropy regularization (cf. Section \ref{sec:formulation}):
\begin{equation}
\max_{\pi\in \Pi} V_\Lamb^\pi(\rho).
\end{equation}
As shown in \citep{nachum2017bridging}, $\pi_\Lamb$ can be uniquely characterized as
\begin{equation}
\pi_\Lamb(a \vert s)\propto \exp \left(\frac{Q_{\Lamb}^{\pi_\Lamb}(s, a)-V_{\Lamb}^{\pi_\Lamb}(s)}{\tau}\right),\ \forall (s,a)\in \mathcal{S}\times \mathcal{A}.
\end{equation}
Therefore, a standard result in the duality theory \citep{floudas1995nonlinear} implies that $D(\Lamb)$ is differentiable with the gradient
\begin{equation}
\nabla D(\Lamb) = U_\mb{g}^{\pi_\Lamb}(\Lamb) - \mb{b}= \left(U_{g_1}^{\pi_\Lamb}(\Lamb) - b_1, \dots, U_{g_n}^{\pi_\Lamb}(\Lamb) - b_n\right).
\end{equation}

Next, we show that $\nabla D(\Lamb)$ is Lipschitz continuous on $\Lambda$, which implies smoothness.
Consider the statements:
\begin{enumerate}
    \item $(U_{g_i}^{\pi}(\rho)-b_i)$ is $\ell_c$-Lipschitz continuous with respect to $\pi$, which is already proved in Lemmma \ref{lemma:lipschitz value function}.
    \item $\pi_\Lamb$ is $\ell_\Lambda$-Lipschitz continuous with respect to $\Lamb$ for some $\ell_\Lambda >0$.
\end{enumerate}
If these statements hold true, it follows that
\begin{equation}
\left|(U_{g_i}^{\pi_{\Lamb_1}}(\rho)-b_i)-(U_{g_i}^{\pi_{\Lamb_2}}(\rho)-b_i)\right| \leq \ell_c \|\pi_{\Lamb_1}-\pi_{\Lamb_2}\|_2 \leq \ell_c\ell_\Lambda \|\Lamb_1-\Lamb_2\|_2,
\end{equation}
which leads to
\begin{equation}
\left\|\nabla D(\Lamb_1)-\nabla D(\Lamb_2)\right\|_2 \leq \sqrt{n}\ell_c\ell_\Lambda \|\Lamb_1-\Lamb_2\|_2,
\end{equation}
i.e. $D(\Lamb)$ is $\sqrt{n}\ell_c\ell_\Lambda$-strongly smooth on $\Lambda$.

To prove Statement 2, consider the Lagrangian $L(\pi, \Lamb)$, which is twice continuously differentiable on $(0,1)^{|\mathcal{S}|\times |\mathcal{A}|} \times \Lambda$.
The hidden constraint for the maximization problem (\ref{def:dual function}) is linear and has the form
\begin{equation}
\sum_{a\in \mathcal{A}} \pi(a \vert s) = 1,\ \forall s\in \mathcal{S}.
\end{equation}
By Proposition \ref{prop:quadratic lower bound}, it holds that
\begin{equation}
\nabla_{\pi\pi} L\left(\pi_\Lamb, \Lamb\right) \succeq \frac{\tau d 
}{2(1-\gamma)\ln 2}I_{|\mathcal{S}||\mathcal{A}|},
\end{equation}
which implies that $\pi_\Lamb$ is a strict global maximizer of order 2 under Assumption \ref{assump:initial distribution}.

Consider $\nabla^2_{\pi\Lamb} L(\pi,\Lamb)$, which is a matrix of dimension $|\mathcal{S}||\mathcal{A}|\times n$.
Specifically, it holds
\begin{equation}
\frac{\partial^2 L(\pi,\Lamb)}{\partial\pi(a\vert s)\partial \lambda_i} \overset{(i)}{=} \frac{\partial}{\partial\pi(a\vert s) } \left(U_{g_i}^\pi(\rho) - b_i\right) \overset{(ii)}{=} \frac{1}{1-\gamma} d_{\rho}^{\pi}(s) Q_{g_i}^{\pi}(s, a),
\end{equation}
where $(i)$ follows from definition (\ref{def:lagrangian}) and and $(ii)$ is due to the policy gradient (cf. Lemma \ref{lemma:policy gradient_direct}).

Following the same argument as in the proof of Proposition \ref{lemma:lipschitz value function_app}, we have 
\begin{equation}
\left\|\nabla^2_{\pi\Lamb} L(\pi,\Lamb)\right\|_F \leq \frac{\sqrt{n|\mathcal{A}|}}{(1-\gamma)^2}.
\end{equation}
Therefore, applying Proposition \ref{prop:parametric optimization} with 
\begin{equation}
w_1 = \frac{\tau d 
}{2(1-\gamma)\ln 2},\ w_2 = \frac{\sqrt{n|\mathcal{A}|}}{(1-\gamma)^2}+1,
\end{equation}
we conclude that $\pi_\Lamb$ is locally $\ell_\Lambda$-Lipschitz continuous with respect to $\Lamb$ for all $\Lamb \in \Lambda$, where $\ell_\Lambda = {w_2}/{w_1}$. 
Since $\ell_\Lambda$ is universal and does not depend on $\Lamb$, the local Lipschitz property is ready to be extended to $\Lambda$.
The proof is completed by setting $\ell =  \sqrt{n}\ell_c\ell_\Lambda$.
\endproof

\begin{proposition}[Restatement of Proposition \ref{prop:dual to primal}]\label{prop:dual to primal_app}
Suppose that Assumptions \ref{assump:slater}, \ref{assump:strong duality}, and \ref{assump:initial distribution} hold. If $\Lamb\geq 0$ is an $\varepsilon$-optimal multiplier, i.e. $D(\Lamb) -D^\star_\tau \leq \varepsilon$, then the associated Lagrangian maximizer $\pi_\Lamb$ satisfies
\begin{subequations}
\begin{align}
\left\|\pi_\Lamb - \pi_\tau^\star\right\|_2 &\leq C_1 \sqrt{\varepsilon},\label{eq:pi optimality_3_app}\\
\left |V_\tau^{\pi_\Lamb}(\rho)- V_\tau^\star\right| &\leq 2\varepsilon+\ell_cC_1C_2\sqrt{\varepsilon},\label{eq:primal optimality_3_app}\\
\max_{i\in[n]} \left[b_i - U_{g_i}^{\pi_\Lamb}(\rho) \right]_+ &\leq \ell_cC_1 \sqrt{\varepsilon},\label{eq:constraint violation_3_app}
\end{align}
\end{subequations}
where 
\begin{equation}
\ell_c = \frac{\sqrt{|\mathcal{A}|}}{(1-\gamma)^2},\ C_1 = \sqrt{\frac{2(1-\gamma) \ln 2}{\tau d }},\ C_2 = \left(V_\tau^\star-V_{\tau}^{\overline{\pi}}(\rho) \right)\left(\sum_{i=1}^n \frac{1}{\xi_i}\right).
\end{equation}
\end{proposition}

\begin{proof}
We can write $D(\Lamb) = L\left(\pi_\Lamb, \Lamb\right)$ and $D^\star_\tau = L\left(\pi_\tau^\star, \Lamb^\star\right)$, where $\left(\pi_\tau^\star, \Lamb^\star\right)$ is any primal-dual pair.
Then, by the strong duality, we have
\begin{equation}
L\left(\pi_\tau^\star, \Lamb^\star\right)= \min_{\muu \geq 0} L\left(\pi_\tau^\star, \muu\right)\leq L\left(\pi_\tau^\star, \Lamb\right).
\end{equation}
Therefore, 
\begin{equation}\label{eq:pi_Lamb bound computation_app}
\begin{aligned}
\varepsilon \geq L\left(\pi_\Lamb, \Lamb\right) - 	L\left(\pi_\tau^\star, \Lamb^\star\right) 
&= L\left(\pi_\Lamb, \Lamb\right) -\min_{\muu \geq 0} L\left(\pi_\tau^\star, \muu\right)\\
&\geq L\left(\pi_\Lamb, \Lamb\right) - L\left(\pi_\tau^\star, \Lamb\right) \\
&{\overset{(i)}{\geq}} \frac{\tau d 
}{2(1-\gamma)\ln 2}\left\|\pi_\Lamb - \pi_\tau^\star\right\|_2^2,
\end{aligned}
\end{equation}
where $(i)$ results from the quadratic lower bound given by Proposition \ref{prop:quadratic lower bound}.
Then, (\ref{eq:pi optimality_3_app}) is obtained after rearranging the terms in (\ref{eq:pi_Lamb bound computation_app}).

Next, we can use the Lipschitz continuity of the utility function (cf. Lemma \ref{lemma:lipschitz value function}) to bound the constraint violation.
For every $i = 1,2,\dots, n$, it holds that
\begin{equation}\label{eq:dualgap_component_app1}
\left|U_{g_i}^{ \pi_{\Lamb}}(\rho)-U_{g_i}^{\pi_\tau^\star}(\rho)\right| \leq 
\ell_c\left\|\pi_{\Lamb}- \pi_\tau^\star\right\|_2 
\leq \ell_cC_1 \sqrt{\varepsilon}.
\end{equation}
As the optimal policy $\pi_\tau^\star$ must be feasible to (\ref{prob:maxentrl}), i.e. $U_\mb{g}^{\pi_\tau^\star}(\rho) \geq \mb{b}$, we can bound the constraint violation as
\begin{equation}
\max_{i\in [n]} \left[b_i - U_{g_i}^{ \pi_{\Lamb}}(\rho) \right]_+
\leq \max_{i\in [n]} \left\{\left[b_i - U_{g_i}^{\pi_\tau^\star} (\rho) \right]_+ + \left|U_{g_i}^{ \pi_{\Lamb}}(\rho)-U_{g_i}^{\pi_\tau^\star}(\rho)\right| \right\}
\leq \ell_cC_1 \sqrt{\varepsilon}.
\end{equation}

Finally, to bound the primal optimality gap, we note that
\begin{equation}
0\overset{(i)}\leq L\left(\pi_\tau^\star,\Lamb \right) - L\left(\pi_\tau^\star,\Lamb^{\star} \right) \overset{(ii)}\leq L\left(\pi_{\Lamb},\Lamb \right) - L\left(\pi_\tau^\star,\Lamb^{\star}\right) = D\left(\Lamb \right)-D\left(\Lamb^\star\right) \leq \varepsilon,
\end{equation}
where $(i)$ follows from the strong duality and $(ii)$ is due to the definition of $\pi_{\Lamb}$ (cf. (\ref{eq:pi_lambda evaluation})).
Thus, by expanding the Lagrangian as
\begin{equation}
\begin{aligned}
L\left(\pi_\tau^\star,\Lamb \right) - L\left(\pi_\tau^\star,\Lamb^{\star} \right) 
&= V_\tau^{\pi_\tau^\star}(\rho) + \Lamb^\top\left(U_\mb{g}^{\pi_\tau^\star}(\rho) -\mb{b}\right) - V_\tau^{\pi_\tau^\star}(\rho) - \left(\Lamb^\star\right)^\top\left(U_\mb{g}^{\pi_\tau^\star}(\rho) -\mb{b}\right)\\
&= \left(\Lamb - \Lamb^\star\right)^\top\left(U_\mb{g}^{\pi_\tau^\star}(\rho)-\mb{b} \right),
\end{aligned}
\end{equation}
and applying the complementary slackness $\left(\Lamb^\star\right)^\top\left(U_\mb{g}^{\pi_\tau^\star}(\rho)-b \right) = 0$, we obtain the bound
\begin{equation}\label{eq:intermediate bound_app1}
0\leq \left(\Lamb\right)^\top\left(U_\mb{g}^{\pi_\tau^\star}(\rho)-b \right) \leq \varepsilon.   
\end{equation}
Therefore, 
\begin{equation}\label{eq:bound_LambG_app}
\begin{aligned}
\left|\left(\Lamb\right)^\top\left(U_\mb{g}^{\pi_{\Lamb}}(\rho)-b \right)\right|
&\overset{(i)}\leq \left|\left(\Lamb\right)^\top\left(U_\mb{g}^{\pi_\tau^\star}(\rho)-b \right)\right|+\left| \left(\Lamb\right)^\top\left(U_\mb{g}^{\pi_{\Lamb}}(\rho)- U_\mb{g}^{\pi_\tau^\star}(\rho)\right)\right|\\
&\overset{(ii)}\leq \varepsilon +  \ell_cC_1\left(V_\tau^\star-V_{\tau}^{\overline{\pi}}(\rho) \right)\left(\sum_{i=1}^n \frac{1}{\xi_i}\right)\sqrt{\varepsilon},
\end{aligned}
\end{equation}
where $(i)$ is due to the triangular inequality and $(ii)$ uses the bound (\ref{eq:dualgap_component_app1}) and the boundedness of $\Lambda$ (cf. Lemma \ref{lemma:dual boundedness}), i.e. $0 \leq \lambda_i \leq\left(V_\tau^\star-V_{\tau}^{\overline{\pi}}(\rho)\right) / \xi_i$ for all $i\in [n]$ and $\Lamb\in \Lambda$.
Thus, we can bound the primal optimality gap as
\begin{equation}
\begin{aligned}
\left|V_\tau^{\pi_{\Lamb}}(\rho) - V_\tau^\star\right| 
&= \left|V_\tau^{\pi_{\Lamb}}(\rho) - V_\tau^{\pi_\tau^\star}(\rho)\right|\\
&{\overset{(i)}=}  \left|\left[V_\tau^{\pi_{\Lamb}}(\rho) + \left(\Lamb\right)^\top\left(U_\mb{g}^{\pi_{\Lamb}}(\rho) -b \right)\right] -\left(\Lamb\right)^\top\left(U_\mb{g}^{\pi_{\Lamb}}(\rho) -b \right)\right. \\
&\quad -\left.\left[V_\tau^{\pi_\tau^\star}(\rho) + \left(\Lamb^\star\right)^\top\left(U_\mb{g}^{\pi_\tau^\star}(\rho) -b \right)  \right]\right|\\
&{\overset{(ii)}\leq}  \left|L\left(\pi_{\Lamb}, \Lamb \right) - L\left(\pi_\tau^\star, \Lamb^\star\right) \right| + \left|\left(\Lamb\right)^\top\left(U_\mb{g}^{\pi_{\Lamb}}(\rho) -b \right) \right|\\
&{\overset{(iii)}\leq}  \varepsilon + 
\left(\varepsilon +  \ell_cC_1\left(V_\tau^\star-V_{\tau}^{\overline{\pi}}(\rho) \right)\left(\sum_{i=1}^n \frac{1}{\xi_i}\right) \sqrt{\varepsilon}\right)\\
&= 2\varepsilon + \ell_cC_1\left(V_\tau^\star-V_{\tau}^{\overline{\pi}}(\rho) \right)\left(\sum_{i=1}^n \frac{1}{\xi_i}\right)\sqrt{\varepsilon}\\
& = 2\varepsilon + \ell_cC_1C_2\sqrt{\varepsilon},
\end{aligned}
\end{equation}
where $(i)$ uses the complementary slackness $\left(\Lamb^\star\right)^\top\left(U_\mb{g}^{\pi_\tau^\star}(\rho) -b \right) = 0$ and $(ii)$ uses the triangular inequality and the definition of Lagrangian (\ref{def:lagrangian}).
In $(iii)$, we use the assumption 
\begin{equation}
D(\Lamb) - D\left(\Lamb^\star\right)= L\left(\pi_{\Lamb}, \Lamb \right) - L\left(\pi_\tau^\star, \Lamb^\star\right) \leq \varepsilon,
\end{equation}
and the inequality (\ref{eq:bound_LambG_app}).
This completes the proof.
\end{proof}

 \section{Supplementary Materials for Sections \ref{sec:alg} and \ref{sec:convergence}}\label{app:2}
\subsection{Entropy-regularized NPG}\label{subsec:entropy-regularized npg_app}
For entropy-regularized MDPs, the natural policy gradient update rule can be written as
\begin{equation}
\theta \leftarrow \theta+\eta\left(\mathcal{F}_{\rho}^{\theta}\right)^{\dagger} \nabla_{\theta} V_{\tau}^{\pi_{\theta}}(\rho),
\end{equation}
where $\mathcal{F}_{\rho}^{\theta}$ is the Fisher information matrix, defined as
\begin{equation}
\mathcal{F}_{\rho}^{\theta}:=\underset{
s \sim d_{\rho}^{\pi_{\theta}},
a \sim \pi_{\theta}(\cdot \vert s)     
 }{\mathbb{E}}\left[\left(\nabla_{\theta} \log \pi_{\theta}(a \vert s)\right)\left(\nabla_{\theta} \log \pi_{\theta}(a \vert s)\right)^{\top}\right].
\end{equation}
Under the soft-max parameterization, the associated policy update has a fairly direct form (cf. (\ref{eq:policy update}) and (\ref{eq:Zt})).
We refer the reader to \citep{cen2021fast} for a detailed derivation.

\citet{cen2021fast} proved that the entropy-regularized NPG method enjoys a global linear convergence and a local quadratic convergence.
We summarize the two results in Propositions \ref{prop:cen_app} and \ref{prop:cen_quadratic_app}, where we abuse the notations and denote the optimal unconstrained value function with entropy regularization, the corresponding Q-function, and the associated optimal policy with $V^\star_\tau$, $Q^\star_\tau$, and $\pi^\star_\tau$ respectively.
Let $\mu_{\tau}^{\star}$ denote the stationary distribution over $\mc{S}$ of the MDP under policy $\pi_\tau^\star$\footnote{It is straightforward to verify that $d_{\mu_{\tau}^{\tau}}^{\pi_{\tau}^{\star}}=\mu_{\tau}^{\star}$.}.

\begin{proposition}[Global linear convergence]\label{prop:cen_app}
If the step-size $\eta= (1-\gamma)/\tau$ is used, the entropy-regularized NPG algorithm (\ref{eq:policy update}) satisfies the error bounds:
\begin{subequations}\label{eq:exactnpg}
\begin{align}
\left\|V_{\tau}^{\star}-V_{\tau}^{\pi^{(t+1)}}\right\|_{\infty} & \leq 3 \left\|Q_{\tau}^{\star}-Q_{\tau}^{\pi^{(0)}}\right\|_{\infty} \gamma^{t+1},\\
\left\|\log \pi^{\star}_\tau-\log \pi^{(t+1)}\right\|_{\infty} & \leq 2 \left\|Q_{\tau}^{\star}-Q_{\tau}^{\pi^{(0)}}\right\|_{\infty} \tau^{-1}\gamma^{t},
\end{align}
\end{subequations}
for all $t >0$, where
\begin{subequations}
\begin{align}
\left\|V_{\tau}^{\star}-V_{\tau}^{\pi^{(t+1)}}\right\|_{\infty} &:= \max_{s\in\mc{S}} \left| V_{\tau}^{\star}(s)-V_{\tau}^{\pi^{(t+1)}}(s)\right|,\\
\left\|Q_{\tau}^{\star}-Q_{\tau}^{\pi^{(0)}}\right\|_{\infty}&:= \max_{s\in\mc{S},a\in\mc{A}} \left|Q_{\tau}^{\star}(s,a)-Q_{\tau}^{\pi^{(0)}}(s,a)\right|,\\
\left\|\log \pi^{\star}_\tau-\log \pi^{(t+1)}\right\|_{\infty}&:= \max_{s\in\mc{S},a\in\mc{A}}\left|\log \pi^{\star}_\tau(a\vert s)-\log \pi^{(t+1)}(a\vert s)\right|.
\end{align}
\end{subequations}
\end{proposition}
We note that \citet{cen2021fast} proved a more general result for all step-sizess $\eta \in [0,(1-\gamma)/\tau]$, whereas the fastest convergence is achieved with the maximum step-size $\eta = (1-\gamma)/\tau$.
\begin{proposition}[Local quadratic convergence]\label{prop:cen_quadratic_app}
{Suppose that the entropy-regularized NPG algorithm (\ref{eq:policy update}) with the step-size $\eta = (1-\gamma)/\tau$ satisfies 
\begin{equation}
\left\|\log \pi^{(t)}-\log \pi_{\tau}^{\star}\right\|_{\infty} \leq 1,
\end{equation}
for all $t\geq 0$. There exist problem-dependent constants $K_1$ and $K_2$ such that
\begin{equation}
V_{\tau}^{\star}(\rho)-V_{\tau}^{(t)}(\rho) \leq K_1\left(K_2 \left(V_{\tau}^{\star}\left(\mu_{\tau}^{\star}\right)-V_{\tau}^{\pi^{(0)}}\left(\mu_{\tau}^{\star}\right)\right)\right)^{2^t}.
\end{equation}}
\end{proposition}

In our work, $V_\Lamb^\pi(\rho)$ is the entropy-regularized value function associated with the Lagrangian $L(\pi,\Lamb)$, which has the reward function $r_\Lamb(s,a) = r(s,a) + \Lamb ^\top \mb{g}(s,a)$.
Therefore, Proposition \ref{prop:cen_app} implies that, with step-size $\eta = (1-\gamma)/\tau$, the error bound $\left\|\log \pi_{\Lamb}-\log \pi^{(t)}\right\|_{\infty} \leq \varepsilon$ can be achieved in 
\begin{equation}\label{eq:npg complexity_app}
\frac{1}{1-\gamma} \log \left(\frac{2\left\|Q_\Lamb^{\pi_\Lamb}-Q_{\Lamb}^{\pi^{(0)}}\right\|_{\infty}}{\varepsilon\tau}\right),
\end{equation}
iterations (cf. (\ref{eq:npg complexity})).
Furthermore, since $\Lamb \in \Lambda =  \{\Lamb \mid 0\leq \lambda_i \leq \left(V_\tau^\star-V_{\tau}^{\overline{\pi}}(\rho)\right) / \xi_i,\text{ for all } i \in [n]\}$, we have that $r_\Lamb(s,a) \in [0,1+C_2]$, where $C_2 = \left(V_\tau^\star-V_{\tau}^{\overline{\pi}}(\rho) \right)\left(\sum_{i=1}^n {1}/{\xi_i}\right)$.
Together with the elementary entropy bound $\mathcal{H}(\rho, \pi) \in \left[0,\log|\mathcal{A}|/(1-\gamma)\right]$, it holds that 
\begin{equation}\label{eq:Q lambda bound app}
Q_\Lamb^\pi(s,a)\in \left[0,\frac{1+C_2+\tau\log|\mathcal{A}|}{1-\gamma} \right],
\end{equation}
for all $\Lamb\in \Lambda$.
Thus, we can drop the dependency of Q-function in (\ref{eq:npg complexity_app}) to obtain the following bound on the number of iterations:
\begin{equation}\label{eq:npg complexity_app_new}
\frac{1}{1-\gamma} \log \left(\frac{2\left(1+C_2+\tau\log |\mathcal{A}|\right)}{\varepsilon\tau(1-\gamma)}\right).
\end{equation}

\subsection{Accelerated Gradient Projection Method with Inexact Gradient}\label{subsec:gradient projection}
Gradient projection method is a feasible direction method for solving constrained optimization problems of the form:
\begin{equation}\label{prob:constrained opt_app}
\begin{aligned}
\min_x & \ f(x)\\
s.t. & \ x \in X
\end{aligned}
\end{equation}
where $f(x)$ is convex and differentiable and $X$ is convex.
The general update scheme is
\begin{equation}\label{eq:gradient projection_app}
x^{(k+1)} = \mathcal{P}_X\left(x^{(k)}-\alpha^k\nabla f\left(x^{(k)}\right)\right).
\end{equation}
When the gradient is inexact, \citet{schmidt2011convergence} proved the following bound for the general update (\ref{eq:gradient projection_app}) and the accelerated update (\ref{eq:accelerated gradient projection}).
\begin{proposition}[Convergence of inexact gradient projection method]\label{prop:gradient projection_app}
Assume that $f(x)$ is convex and $L$-smooth on $X$, and that we have access to a gradient oracle $h(x)$ such that $\left\|\nabla f(x) - h(x) \right\|_2 \leq \delta$ for all $x\in X$.
Let $x^\star = \operatorname{arg}\min_{x\in X}f(x)$.
By selecting $\alpha_k = {1}/{L}$ and $\beta_k = {(k-1)}/{(k+2)}$, then the iterates of algorithm (\ref{eq:gradient projection_app}) satisfy
\begin{equation}
f\left(\frac{1}{k} \sum_{i=1}^{k} x^{(i)}\right)-f\left(x^{\star}\right) \leqslant \frac{L}{2 k}\left(\left\|x^{(0)}-x^{\star}\right\|_2+\frac{2k\delta}{L}\right)^{2} = O\left(\frac{1}{k} \right) + O\left(k^2\delta^2 \right)+O\left(k\delta \right).
\end{equation}
Moreover, for the accelerated version (\ref{eq:accelerated gradient projection}), it holds that
\begin{equation}
f\left(x^{(k)}\right)-f\left(x^{\star}\right) \leqslant \frac{2L}{(k+1)^2}\left(\left\|x^{(0)}-x^{\star}\right\|_2+\frac{(k+1)k\delta}{L}\right)^{2} = O\left(\frac{1}{k^2} \right) + O\left(k^2\delta^2 \right)+O\left(\delta \right).
\end{equation}
\end{proposition}

\subsection{Proofs of Results in Section \ref{sec:convergence}}\label{subsec:convergence proof_app}
\begin{proposition}[Restatement of Proposition \ref{prop:gradient evaluation}]\label{prop:gradient evaluation_app}
Suppose that $\pi$ is an approximate solution to (\ref{eq:pi_lambda evaluation}) such that $\|\log \pi - \log \pi_\Lamb\|_\infty \leq \varepsilon$.  
The gradient estimator defined by 
\begin{equation}
\widetilde \nabla D\left(\Lamb\right) :=U_\mb{g}^{\pi}(\rho) -\mb{b} =  \left(U_{g_1}^{\pi}(\rho) -b_1,\dots,U_{g_n}^{\pi}(\rho) -b_n  \right),
\end{equation}
satisfies
\begin{equation}
\left\|\widetilde \nabla D\left(\Lamb\right) - \nabla D\left(\Lamb\right)\right\|_2  \leq \frac{\sqrt{n}|\mathcal{A}|}{(1-\gamma)^2}\varepsilon.
\end{equation}
\end{proposition}
\begin{proof}
Since $\left(\log x \right)'  = {1}/{x} \geq 1$ for all $x\in (0,1]$, it holds that $\|\pi-\pi_\Lamb\|_\infty \leq \|\log \pi - \log \pi_\Lamb\|_\infty \leq \varepsilon$.
To bound the quantity $\left\|\widetilde \nabla D\left(\Lamb\right) - \nabla D\left(\Lamb\right)\right\|_2 = \left\|U_\mb{g}^\pi (\rho) - U_\mb{g}^{\pi_\Lamb} (\rho) \right\|_2$, we can either use the Lipschitz continuity of $U_\mb{g}^{\pi}(\rho)$ (cf. Lemma \ref{lemma:lipschitz value function}) or use the performance difference lemma (cf. Lemma \ref{lemma:performance difference}).

With the Lipschitz continuity, we have
\begin{equation}\label{eq:lipschitz method_app}
\left|U_{g_i}^\pi(\rho) - U_{g_i}^{\pi_\Lamb}(\rho)\right| \leq \ell_c \|\pi-\pi_\Lamb\|_2 \leq \ell_c\sqrt{n}\|\pi-\pi_\Lamb\|_\infty = \frac{\sqrt{n|\mathcal{A}|}}{(1-\gamma)^2}\varepsilon,
\end{equation}
where we have used the inequality $\|\cdot\|_2\leq \sqrt{n}\|\cdot\|_\infty$.
Therefore,
\begin{equation}\label{eq:lipschitz method result_app}
\left\|\widetilde \nabla D\left(\Lamb\right) - \nabla D\left(\Lamb\right)\right\|_2
\leq \sqrt{n}\left\|\widetilde \nabla D\left(\Lamb\right) - \nabla D\left(\Lamb\right)\right\|_\infty =\sqrt{n}\left\|U_\mb{g}^\pi (\rho) - U_\mb{g}^{\pi_\Lamb} (\rho) \right\|_\infty \leq \frac{ n\sqrt{|\mathcal{A}|}}{(1-\gamma)^2}\varepsilon.
\end{equation}
On the other hand, we can use the performance difference lemma to obtain
\begin{equation}\label{eq:performance difference method_app}
\begin{aligned}
\left|U_{g_i}^\pi(\rho) - U_{g_i}^{\pi_\Lamb}(\rho)\right| &= \left|\frac{1}{1-\gamma} \sum_{s\in \mathcal{S}} d_\rho^{\pi_\Lamb}(s) \sum_{a\in \mathcal{A}} \left(\pi(a\vert s) - \pi_\Lamb(a\vert s)\right) Q_{g_i}^{\pi}(s,a)\right|\\
& \leq \frac{1}{1-\gamma} \sum_{s\in \mathcal{S}} d_\rho^{\pi_\Lamb}(s) \sum_{a\in \mathcal{A}} \left|\pi(a\vert s) - \pi_\Lamb(a\vert s)\right| Q_{g_i}^{\pi}(s,a)\\
&\overset{(i)}\leq \frac{\varepsilon}{1-\gamma} \sum_{s\in \mathcal{S}} d_\rho^{\pi_\Lamb}(s) \sum_{a\in \mathcal{A}} Q_{g_i}^{\pi}(s,a)\\
&\overset{(ii)}\leq  \frac{|\mathcal{A}|}{(1-\gamma)^2}\varepsilon,
\end{aligned}
\end{equation}
where $(i)$ is based on the bound $\|\pi-\pi_\Lamb\|_\infty\leq \varepsilon	$ and $(ii)$ is due to $Q^\pi_{g_i}(s,a)\leq {1}/({1-\gamma})$ and the fact that $d_\rho^{\pi_\Lamb}(\cdot)$ is a probability distribution.
Repeating (\ref{eq:lipschitz method result_app}) with the bound (\ref{eq:performance difference method_app}) yields that
\begin{equation}\label{eq:performance difference method result_app}
\left\|\widetilde \nabla D\left(\Lamb\right) - \nabla D\left(\Lamb\right)\right\|_2
\leq \sqrt{n}\left\|\widetilde \nabla D\left(\Lamb\right) - \nabla D\left(\Lamb\right)\right\|_\infty =\sqrt{n}\left\|U_\mb{g}^\pi (\rho) - U_\mb{g}^{\pi_\Lamb} (\rho) \right\|_\infty \leq \frac{\sqrt{n}|\mathcal{A}|}{(1-\gamma)^2}\varepsilon.
\end{equation}
Equations (\ref{eq:lipschitz method result_app}) and (\ref{eq:performance difference method result_app}) give two upper bounds on the quantity $\left\|\widetilde \nabla D\left(\Lamb\right) - \nabla D\left(\Lamb\right)\right\|_2$.
In this work, we use the bound (\ref{eq:performance difference method result_app}), as it has a weaker dependence on the number of constraints $n$.
This completes the proof.
\end{proof}
Proposition \ref{prop:gradient evaluation} implies that running Algorithm \ref{alg:npg} with the step-size $\eta = (1-\gamma)/\tau$ for 
\begin{equation}\label{eq:N_2_app}
\frac{1}{1-\gamma}\log \left( \frac{2\sqrt{n} \left|\mathcal{A}\right|\left(1+C_2+\tau\log |\mathcal{A}|\right)}{\delta (1-\gamma)^3\tau} \right).
\end{equation}
iterations, where $C_2 = \left(V_\tau^\star-V_{\tau}^{\overline{\pi}}(\rho) \right)\left(\sum_{i=1}^n {1}/{\xi_i}\right)$, guarantees a $\delta$-accurate gradient estimation $\widetilde \nabla D\left(\Lamb\right)$, i.e.  $\left\|\widetilde \nabla D\left(\Lamb\right) - \nabla D\left(\Lamb\right)\right\|_2 \leq \delta$ (cf. (\ref{eq:N_2})).

\begin{theorem}[Restatement of Theorem \ref{thm:convergence}]\label{thm:convergence_app}
Suppose that Assumptions \ref{assump:slater}, \ref{assump:strong duality}, and \ref{assump:initial distribution} hold.  For every $\varepsilon_1>0$, Algorithm \ref{alg:dualpg} with a random initialization and the parameters $\eta = (1-\gamma)/\tau$, $\alpha_t = {1}/{\ell}$, $\beta_t = {(t-1)}/{(t+2)}$, and 
\begin{equation}
N_1 = T,\ N_2 = \frac{1}{1-\gamma}\log \left( \frac{2\sqrt{n} \left|\mathcal{A}\right|T(T+1)\left(1+C_2+\tau\log|\mathcal{A}|\right)}{ (1-\gamma)^3\tau\ell} \right),\ N_3 = \frac{1}{1-\gamma} \log \left(\frac{2\sqrt{n}\left(1+C_2+\tau\log|\mathcal{A}|\right)}{\varepsilon_1 \tau(1-\gamma)}\right),
\end{equation}
returns a solution pair $(\pi, \Lamb)$ such that
\begin{subequations}
\begin{align}
D(\Lamb) - D_\tau^\star &\leq \varepsilon_0,\label{eq:thm results_app_0} \\
\left\|\pi - \pi_\tau^\star\right\|_2 &\leq C_1\sqrt{\varepsilon_0} + \varepsilon_1,\label{eq:thm results_app_1}\\
\left|V_\tau^{\pi}(\rho) - V_\tau^\star\right| &\leq 2\varepsilon_0 + \ell_cC_1C_2\sqrt{\varepsilon_0} + \left(\ell_cC_2 + \frac{3\gamma}{2\tau\sqrt{n}}\right)\varepsilon_1,\label{eq:thm results_app_3}\\
\max_{i\in [n]} \left[b_i - U_{g_i}^{\pi}(\rho) \right]_+ &\leq \ell_c\left(C_1\sqrt{\varepsilon} +\varepsilon_1 \right),\label{eq:thm results_app_2}
\end{align}
\end{subequations}
where 
\begin{equation}\label{eq:constants epsilon ell_app}
\varepsilon_0 = \frac{2\ell}{(T+1)^2}\left(\left\|\Lamb^{(0)}-\Lamb^{\star}\right\|_2+1\right)^{2},\ \ell = \frac{2\ln 2 \left(n|\mathcal{A}|+(1-\gamma)^2\sqrt{n|\mathcal{A}|}\right)}{\tau (1-\gamma)^3 d},
\end{equation}
and
\begin{equation}
\ell_c = \frac{\sqrt{|\mathcal{A}|}}{(1-\gamma)^2},\ C_1 = \sqrt{\frac{2 (1-\gamma)\ln 2}{\tau d }},\ C_2 = \left(V_\tau^\star-V_{\tau}^{\overline{\pi}}(\rho) \right)\left(\sum_{i=1}^n \frac{1}{\xi_i}\right).
\end{equation}
The total iteration complexity is $N_1\times N_2 +N_3 = \widetilde{\mathcal{O}}(T)$ with primal error bounds $\mathcal{O}\left({1}/{T}\right)$ given by (\ref{eq:thm results_app_1})-(\ref{eq:thm results_app_2}) and a dual error bound $\mathcal{O}\left({1}/{T^2}\right)$ given by (\ref{eq:thm results_app_0}).
\end{theorem}
\begin{proof}
Under Assumptions \ref{assump:slater} and \ref{assump:initial distribution}, it follows from Proposition \ref{prop:dual smoothness} that $D(\Lamb)$ is convex, differentiable, and $\ell$-smooth on $\Lambda$.
Now, we fix the gradient accuracy as
\begin{equation}
\delta = \frac{\ell}{T(T+1)}.
\end{equation}
Then, it follows from Proposition \ref{prop:gradient evaluation} and (\ref{eq:N_2_app}) that running the NPG subroutine in line \ref{alg:npg line} of Algorithm \ref{alg:dualpg} for 
\begin{equation}
N_2 = \frac{1}{1-\gamma}\log \left( \frac{2\sqrt{n} \left|\mathcal{A}\right|T(T+1)\left(1+C_2+\tau\log|\mathcal{A}|\right)}{ (1-\gamma)^3\tau\ell} \right),
\end{equation}
iterations guarantees obtaining an estimation $\widetilde \nabla D\left(\Lamb\right)$ such that
\begin{equation}\label{eq:gradient estimation proof_app}
\left\|\widetilde \nabla D\left(\Lamb\right) - \nabla D\left(\Lamb\right)\right\|_2 \leq \frac{\ell}{T(T+1)},
\end{equation}
where we have use the bound $\left\|Q_\Lamb^{\pi_\Lamb} -Q_\Lamb^{(0)}\right\|_\infty \leq \left({1+C_2+\log|\mc{A}|}\right)/({1-\gamma})$ for all $\Lamb\in \Lambda$ (cf. (\ref{eq:Q lambda bound app})).
Therefore, by Proposition \ref{prop:gradient projection_app}, running the outer loop in Algorithm \ref{alg:dualpg} for $N_1 = T$ iterations generates a solution $\Lamb^{(T)}$ such that 
\begin{equation}\label{eq:dualFWgap}
D\left(\Lamb^{(T)} \right)-D_\tau^\star \leq 
\frac{2\ell}{(T+1)^2}\left(\left\|\Lamb^{(0)}-\Lamb^{\star}\right\|_2+\frac{(T+1)T\delta}{\ell}\right)^{2} 
= \frac{2\ell}{(T+1)^2}\left(\left\|\Lamb^{(0)}-\Lamb^{\star}\right\|_2+1\right)^{2},
\end{equation}
which satisfies (\ref{eq:thm results_app_0}).

Below, we adopt the proof of Proposition \ref{prop:dual to primal} (cf. Appendix \ref{app:1}). 
We first apply Proposition \ref{prop:dual to primal} with
\begin{equation}
\varepsilon_0 = \frac{2\ell}{(T+1)^2}\left(\left\|\Lamb^{(0)}-\Lamb^{\star}\right\|_2+1\right)^{2}
\end{equation}
to obtain $\left\|\pi_{\Lamb^{(T)}} - \pi_\tau^\star\right\|_2 \leq C_1 \sqrt{\varepsilon_0}$, where $\pi_\tau^\star$ is an optimal policy.
By Propositions \ref{prop:cen_app}, we can compute an approximate Lagrangian maximizer $\widetilde \pi _{\Lamb^{(T)}}$ to (\ref{eq:pi_lambda evaluation}) such that $\left\|\log\widetilde\pi_{\Lamb^{(T)}} - \log\pi_{\Lamb^{(T)}}\right\|_\infty \leq \varepsilon_1/\sqrt{n}$, by running Algorithm \ref{alg:npg} for 
\begin{equation}
N_3 = \frac{1}{1-\gamma} \log \left(\frac{2\sqrt{n}\left(1+C_2+\tau\log|\mathcal{A}|\right)}{\varepsilon_1 \tau(1-\gamma)}\right),
\end{equation}
iterations (cf. (\ref{eq:npg complexity_app})). Now, we show that $\widetilde\pi_{\Lamb^{(T)}}$ is a solution to (\ref{prob:maxentrl}) satisfying (\ref{eq:thm results_app_1})-(\ref{eq:thm results_app_3}).
Firstly, we have 
\begin{equation}
\left\|\log\widetilde\pi_{\Lamb^{(T)}} - \log\pi_{\Lamb^{(T)}}\right\|_2\leq \sqrt{n}\left\|\log\widetilde\pi_{\Lamb^{(T)}} - \log\pi_{\Lamb^{(T)}}\right\|_\infty \leq \varepsilon_1.
\end{equation}
By applying the triangular inequality and using the strong concavity of the logarithm function on $(0,1]$, it holds that
\begin{equation}
\left\|\widetilde\pi_{\Lamb^{(T)}}- \pi_\tau^\star\right\|_2 \leq \left\|\widetilde\pi_{\Lamb^{(T)}}- \pi_{\Lamb^{(T)}}\right\|_2 + \left\|\pi_{\Lamb^{(T)}}- \pi_{\Lamb^{(T)}}\right\|_2 \leq C_1 \sqrt{\varepsilon_0} + \varepsilon_1.
\end{equation}
Then, we bound the constraint violation.
It follows from the Lipschitz continuity of the utility function (cf. (\ref{eq:lipschitz value function})) that 
\begin{equation}\label{eq:dualgap_component_app2}
\left|U_{g_i}^{\widetilde \pi_{\Lamb^{(T)}}}(\rho)-U_{g_i}^{\pi_\tau^\star}(\rho)\right| \leq \ell_c\left\|\widetilde\pi_{\Lamb^{(T)}}- \pi_\tau^\star\right\|_2 \leq \ell_c\left(C_1 \sqrt{\varepsilon_0} + \varepsilon_1\right),\quad \forall i = 1,2,\dots, n.
\end{equation}
As the optimal policy $\pi_\tau^\star$ must be a feasible solution to (\ref{prob:maxentrl}), i.e. $U_\mb{g}^{\pi_\tau^\star}(\rho) \geq \mb{b}$, the constraint violation is bounded as
\begin{equation}
\max_{i\in [n]} \left[b_i - U_{g_i}^{\widetilde \pi_{\Lamb^{(T)}}}(\rho) \right]_+
\leq \max_{i\in [n]} \left\{\left[b_i - U_{g_i}^{\pi_\tau^\star} (\rho) \right]_+ + \left|U_{g_i}^{ \widetilde \pi_{\Lamb^{(T)}}}(\rho)-U_{g_i}^{\pi_\tau^\star}(\rho)\right| \right\}
\leq \ell_c\left(C_1 \sqrt{\varepsilon_0} + \varepsilon_1\right).
\end{equation}

Finally, to bound the primal optimality gap, we note that
\begin{equation}
0\overset{(i)}\leq L\left(\pi_\tau^\star,\Lamb^{(T)} \right) - L\left(\pi_\tau^\star,\Lamb^{\star} \right) \overset{(ii)}\leq L\left(\pi_{\Lamb^{(T)}},\Lamb^{(T)} \right) - L\left(\pi_\tau^\star,\Lamb^{\star}\right) = D\left(\Lamb^{(T)} \right)-D_\tau^\star \leq \varepsilon_0,
\end{equation}
where $(i)$ follows from the strong duality and $(ii)$ is due to the definition of $\pi_{\lambda^{(T)}}$.
Thus, by expanding the Lagrangian as
\begin{equation}
\begin{aligned}
L\left(\pi_\tau^\star,\Lamb^{(T)} \right) - L\left(\pi_\tau^\star,\Lamb^{\star} \right) 
&= V_\tau^{\pi_\tau^\star}(\rho) + \left(\Lamb^{(T)}\right)^\top\left(U_\mb{g}^{\pi_\tau^\star}(\rho) -\mb{b}\right) - V_\tau^{\pi_\tau^\star}(\rho) - \left(\Lamb^\star\right)^\top\left(U_\mb{g}^{\pi_\tau^\star}(\rho) -\mb{b}\right)\\
&= \left(\Lamb^{(T)} - \Lamb^\star\right)^\top\left(U_\mb{g}^{\pi_\tau^\star}(\rho)-\mb{b} \right),
\end{aligned}
\end{equation}
and applying the complementary slackness $\left(\Lamb^\star\right)^\top\left(U_\mb{g}^{\pi_\tau^\star}(\rho)-b \right) = 0$, we obtain the bound
\begin{equation}\label{eq:intermediate bound_app2}
0\leq \left(\Lamb^{(T)}\right)^\top\left(U_\mb{g}^{\pi_\tau^\star}(\rho)-b \right) \leq \varepsilon_0.   
\end{equation}
Therefore, 
\begin{equation}\label{eq:bound_lambdaG_app_2}
\begin{aligned}
\left|\left(\Lamb^{(T)}\right)^\top\left(U_\mb{g}^{\pi_{\Lamb^{(T)}}}(\rho)-\mb{b} \right)\right|
&{\overset{(i)}\leq} \left|\left(\Lamb^{(T)}\right)^\top\left(U_\mb{g}^{\pi_\tau^\star}(\rho)-b \right)\right|+\left| \left(\Lamb^{(T)}\right)^\top\left(U_\mb{g}^{\pi_{\Lamb^{(T)}}}(\rho)- U_\mb{g}^{\pi_\tau^\star}(\rho)\right)\right|\\
&{\overset{(ii)}\leq} \varepsilon_0 +  \ell_c\left(V_\tau^\star-V_{\tau}^{\overline{\pi}}(\rho) \right)\left(\sum_{i=1}^n \frac{1}{\xi_i}\right)\left(C_1 \sqrt{\varepsilon_0} + \varepsilon_1\right)\\
&= \varepsilon_0 + \ell_cC_2\left(C_1 \sqrt{\varepsilon_0} + \varepsilon_1\right),
\end{aligned}
\end{equation}
where $(i)$ is based on the triangular inequality, 
and $(ii)$ uses the bound (\ref{eq:dualgap_component_app2}) and the boundedness of $\Lambda$, i.e. $0 \leq \lambda_i \leq\left(V_\tau^\star-V_{\tau}^{\overline{\pi}}(\rho)\right) / \xi_i$ for all $i\in [n]$ and $\Lamb\in \Lambda$ (cf. Lemma \ref{lemma:dual boundedness}).
Thus, we can bound the primal optimality gap as
\begin{equation}
\begin{aligned}
\left|V_\tau^{\widetilde\pi_{\Lamb^{(T)}}}(\rho) - V_\tau^\star\right| 
&= \left|V_\tau^{\widetilde\pi_{\Lamb^{(T)}}}(\rho) - V_\tau^{\pi_\tau^\star}(\rho)\right|\\
&{\overset{(i)}=}  \left|\left[V_\tau^{\widetilde\pi_{\Lamb^{(T)}}}(\rho) + \left(\Lamb^{(T)}\right)^\top\left(U_\mb{g}^{\widetilde\pi_{\Lamb^{(T)}}}(\rho) -\mb{b} \right)\right] -\left(\Lamb^{(T)}\right)^\top\left(U_\mb{g}^{\widetilde\pi_{\Lamb^{(T)}}}(\rho) -\mb{b} \right)\right. \\
&\quad - \left.\left[V_\tau^{\pi_\tau^\star}(\rho) + \left(\Lamb^\star\right)^\top\left(U_\mb{g}^{\pi_\tau^\star}(\rho) -b \right)  \right]\right|\\
&{\overset{(ii)}\leq}  \left|L\left(\widetilde\pi_{\Lamb^{(T)}}, \Lamb^{(T)} \right) - L\left(\pi_\tau^\star, \Lamb^\star\right) \right| + \left|\left(\Lamb^{(T)}\right)^\top\left(U_\mb{g}^{\widetilde\pi_{\Lamb^{(T)}}}(\rho) -\mb{b} \right) \right|\\
&{\overset{(iii)}\leq}  \left|L\left(\widetilde\pi_{\Lamb^{(T)}}, \Lamb^{(T)} \right) - L\left(\pi_{\Lamb^{(T)}}, \Lamb^{(T)} \right) \right| + \left|L\left(\pi_{\Lamb^{(T)}}, \Lamb^{(T)} \right) -L\left(\pi_\tau^\star, \Lamb^\star\right) \right|\\
&\quad + \left|\left(\Lamb^{(T)}\right)^\top\left(U_\mb{g}^{\widetilde\pi_{\Lamb^{(T)}}}(\rho) -\mb{b} \right) \right|\\
&{\overset{(iv)}\leq}  \left(\frac{3\gamma}{2\tau\sqrt{n}}\varepsilon_1\right) + (\varepsilon_0) + \left(\varepsilon_0 +  \ell_cC_2\left(C_1 \sqrt{\varepsilon_0} + \varepsilon_1\right)\right)\\
&=  2\varepsilon_0 + \ell_cC_1C_2\sqrt{\varepsilon_0} + \left(\ell_cC_2 + \frac{3\gamma}{2\tau\sqrt{n}}\right)\varepsilon_1,
\end{aligned}
\end{equation}
where $(i)$ uses the complementary slackness $\left(\Lamb^\star\right)^\top\left(U_\mb{g}^{\pi_\tau^\star}(\rho) -b \right) = 0$, $(ii)$ uses the triangular inequality and the definition of Lagrangian (\ref{def:lagrangian}), and $(iii)$ uses the triangular inequality again.
The inequality $(iv)$ contains three parts where the first part uses Proposition \ref{prop:cen_app} as
\begin{equation}
\left|L\left(\widetilde\pi_{\Lamb^{(T)}}, \Lamb^{(T)} \right) - L\left(\pi_{\Lamb^{(T)}}, \Lamb^{(T)} \right) \right| = \left| V_{\Lamb^{(T)}}^{\widetilde\pi_{\Lamb^{(T)}}}(\rho)- V_{\Lamb^{(T)}}^{\pi_{\Lamb^{(T)}}}(\rho) \right| \leq \frac{3\gamma}{2\tau}\left\|\log\widetilde\pi_{\Lamb^{(T)}} - \log\pi_{\Lamb^{(T)}}\right\|_\infty = \frac{3\gamma}{2\tau\sqrt{n}}\varepsilon_1,
\end{equation}
the second part uses the assumption
\begin{equation}
D(\Lamb^{(T)}) - D_\tau^\star=L\left(\pi_{\Lamb^{(T)}}, \Lamb^{(T)} \right) -L\left(\pi_\tau^\star, \Lamb^\star\right) \leq \varepsilon_0,
\end{equation}
and the third part uses the inequality (\ref{eq:bound_lambdaG_app_2}).
This completes the proof.
\end{proof}

\begin{corollary}[Restatement of Corollary \ref{cor:original mdp solution}]\label{cor:original mdp solution_app}
Suppose that Assumptions \ref{assump:slater}, \ref{assump:strong duality}, and \ref{assump:initial distribution} hold. Let 
\begin{equation}
\tau=\frac{(1-\gamma) \varepsilon}{4 \log |\mathcal{A}|}.
\end{equation}
Then, Algorithm \ref{alg:dualpg} computes a solution $\pi$ for the standard CMDP such that
\begin{subequations}
\begin{align}
\left|V^{\pi^\star}(\rho)-V^{\pi}(\rho)\right| & =  \mc{O}{(\varepsilon)},\\
\max_{i\in [n]} \left[b_i - U_{g_i}^{\pi}(\rho) \right]_+ &= \mathcal{O}(\varepsilon),
\end{align}
\end{subequations}
in $\widetilde{\mathcal{O}}\left( 1/\varepsilon^2 \right)$ iterations, where $\pi^\star$ is an optimal policy to the standard CMDP.
\end{corollary}
\begin{proof}
Since the total iteration complexity of Algorithm \ref{alg:dualpg} is dominated by $N_1\times N_2$ and the error bounds (\ref{eq:thm results_1})-(\ref{eq:thm results_2}) are dominated by $\sqrt{\varepsilon_0}$ (cf. (\ref{eq:constants epsilon ell_app})), we ignore the effect of $\varepsilon_0$, $\varepsilon_1$ and only focus on $\sqrt{\varepsilon_0}$ terms in (\ref{eq:thm results_1})-(\ref{eq:thm results_2}) through the analysis below. 

{Firstly, by invoking the optimality of $\pi^\star_\tau$ with respect to the entropy-regularized CMDP and the elementary entropy bound $0 \leq \mathcal{H}(\rho, \pi) \leq  \log |\mathcal{A}|/(1-\gamma)$, we obtain
\begin{equation}\label{eq:sandwich_app}
V^{\pi_{\tau}^{\star}}(\rho)+\frac{\tau}{1-\gamma} \log |\mathcal{A}| \geq V^{\pi_{\tau}^{\star}}(\rho)+\tau \mathcal{H}\left(\rho, \pi_{\tau}^{\star}\right)=V_{\tau}^{\star}(\rho) \geq V_{\tau}^{\pi_{\star}}(\rho) \geq V^{\pi_{\star}}(\rho),
\end{equation}
which implies the sandwich bound (\ref{eq:original and regularized bound}).}
Now, we choose $T$ in such a way that $\ell_cC_1C_2\sqrt{\varepsilon_0}  = \varepsilon/2$, where $\ell_c$, $C_1$, and $C_2$ are specified in Theorem \ref{thm:convergence_app}.
Then, Theorem \ref{thm:convergence_app} implies that running Algorithm \ref{alg:dualpg} with 
\begin{equation}
N_1 = T,\ N_2 = \frac{1}{1-\gamma}\log \left( \frac{2\sqrt{n} \left|\mathcal{A}\right|T(T+1)\left(1+C_2+\tau\log|\mathcal{A}|\right)}{ (1-\gamma)^3\tau\ell} \right)
\end{equation}
returns a solution $\pi$ such that $\left|V_{\tau}^{\star}-V_{\tau}^{\pi}(\rho)\right| = \mathcal{O}(\varepsilon / 2)$ (cf. (\ref{eq:thm results_app_3})). It then follows that 
\begin{equation}\label{eq:primal original bound_app}
\left|V^{\pi^\star}(\rho)-V^{\pi}(\rho)\right| 
{\overset{(i)}{\leq}}\left|V^{\star}(\rho)-V_{\tau}^{\star}(\rho)\right|+\left|V_{\tau}^{\star}(\rho)-V_{\tau}^{\pi}(\rho)\right|+\left|V_{\tau}^{\pi}(\rho)-V^{\pi}(\rho)\right| {\overset{(ii)}{\leq}} \frac{2 \tau \log |\mathcal{A}|}{1-\gamma}+\mathcal{O}\left(\frac{\varepsilon}{2}\right)=\mc{O}(\varepsilon),
\end{equation}
where $(i)$ is due to the triangular inequality and $(ii)$ uses the bound (\ref{eq:sandwich_app}) (cf. (\ref{eq:original and regularized bound})). 
The $\mc{O}(\varepsilon)$-constraint violation follows directly from Theorem \ref{thm:convergence}, since it enjoys the same order of convergence as the primal optimality gap.

We note that $\ell_cC_1C_2\sqrt{\varepsilon_0}$ can be written as ${C}/{(\tau (T+1))}$, where
\begin{equation}
C = \ell_c(C_1\cdot \sqrt{\tau})C_2\left(\sqrt{2\ell \cdot \tau}\right)\left(\|\Lamb^{(0)} - \Lamb^\star\|_2+1 \right)
\end{equation}
is a constant that does not depend on $T$ and $\varepsilon$.
Thus, the choice $\tau=\left[{(1-\gamma) \varepsilon}\right]/\left({4 \log |\mathcal{A}|}\right)$ implies that
\begin{equation}
\frac{C}{T+1} = \frac{(1-\gamma) \varepsilon}{4 \log |\mathcal{A}|}\times \frac{\varepsilon}{2},
\end{equation}
i.e. $T = \mathcal{O}(1/\varepsilon^2)$. Since the total iteration complexity is $N_1\times N_2 = T\times \mc{O}({\log T})$, we obtain the $\widetilde{\mathcal{O}}\left(1/\sqrt{T}\right)$ error bound for the primal optimality gap and the constraint violation.
This completes the proof.
\end{proof}


\section{Supplementary Materals for Section \ref{sec:single constraint}}\label{subsec:bisection_app}
{In this subsection, we consider the special situation where there is a single constraint ($n=1$).
In particular, we use the non-bold notations to emphasize that the associated notations denote numbers instead of vectors used in the previous sections, e.g. multiplier $\lambda$, constraint $U_g^\pi(\rho) \geq b$, Slater condition $V^{\overline{\pi}}_\tau(\rho)-b\geq \xi$.}

Since the feasible region $\Lambda = [0,C_2]$, where $C_2 = \left(V_\tau^\star-V_{\tau}^{\overline{\pi}}(\rho)\right) / \xi$, is bounded
and $D(\lambda)$ is convex, an approximate stationary point is also an approximate optimal solution.
Specifically, if $\left|\nabla D(\lambda)\right| < \varepsilon$, then
\begin{equation}\label{eq:gradient error bound_124_app}
D(\lambda)-D_\tau^\star \leq  \left|\nabla D(\lambda)\right| \times\left|\lambda-\lambda^\star\right|< C_2\varepsilon  .
\end{equation}
Additionally, if $\operatorname{sign}\left(\nabla D(0)\right)= \operatorname{sign}\left(\nabla D(C_2)\right) = 1$, then $D(\lambda)$ attains the optimum at $\lambda = 0$ due to the convexity.
Similarly, $D(\lambda)$ attains the optimum at $\lambda = C_2$ if $\operatorname{sign}\left(\nabla D(0)\right)= \operatorname{sign}\left(\nabla D(C_2)\right) = -1$.
Therefore, it only remains  to consider the case where $D(0)<0$ and $D(C_2)>0$.

{The proposed method aims to find an approximate stationary point with the bisection scheme.
For a given search interval, it computes the gradient at the middle point.
If the gradient is greater than $\varepsilon$, it shrinks the search interval to the left-half interval; if the gradient is smaller than $-\varepsilon$, it shrinks the search interval to the right-half interval.
The iterates terminate when it finds a point $\lambda$ such that $\left|\widetilde{\nabla} D(\lambda)\right|<\varepsilon $, where $\widetilde{\nabla} D(\lambda)$ is the approximate gradient.
We summarize the proposed method in Algorithm \ref{alg:bisection}, where we separately define the gradient estimator (cf. lines \ref{alg:npg line} and \ref{alg:gradient evaluation line} in Algorithm \ref{alg:dualpg}) as a new subroutine Grad$_{Sub}$ (cf. Algorithm \ref{alg:grad}) for the ease of presentation.
We also assume that the initialization is non-trivial in the sense that $\nabla D(0)< 0$ and $\nabla D(C_2)> 0$, since otherwise we can return the optimal solution $\lambda^\star$ as $0$ or $C_2$.}

\begin{algorithm}[tb]
   \caption{Bisection Method with NPG Subroutine\label{alg:bisection}}
\begin{algorithmic}[1]
   \STATE {\bfseries Input:} Initialization $\pi$, $p_0 = 0$, $q_0 = C_2$;
step-size $\eta$; maximum number of iterations $N_1$, $N_2$; threshold $\varepsilon$.
   \FOR{$t = 0,1,2,\dots$}
   \STATE Let $\lambda = (p_t+q_t)/2$.\label{alg_2:shrink}
   \IF{$\left|\text{Grad}_{Sub}\left(\lambda,\pi,\eta,N_1\right)\right|<\varepsilon $}
   \STATE {\bf break}
   \ELSE
   \IF{$\text{Grad}_{Sub}\left(\lambda,\pi,\eta,N_1\right)\geq\varepsilon $}\label{alg_2:line 7}
   \STATE Let $p_{t+1} \leftarrow p_t$ and $q_{t+1}\leftarrow \lambda$.
   \ELSE
   \STATE Let $p_{t+1} \leftarrow \lambda$ and $q_{t+1}\leftarrow q_t$.
   \ENDIF \label{alg_2:line 11}
   \ENDIF
   \ENDFOR
  \STATE Recover the policy from the dual variable: $\widetilde\pi_{\lambda}\leftarrow \text{NPG}_{Sub}\left(\lambda, \pi, \eta, N_2 \right)$. \label{alg:final recover_bi}
\end{algorithmic}
\end{algorithm}

\begin{algorithm}[tb]
   \caption{Gradient Estimator (Grad$_{Sub}$)\label{alg:grad}}
\begin{algorithmic}[1]
   \STATE {\bfseries Input:} Target point $\lambda$, initialization $\pi$, step-size $\eta$, maximum number of iterations $N$.
   \STATE Estimate the optimal policy $\pi_{\lambda}$ for problem (\ref{eq:pi_lambda evaluation}) through the natural policy gradient subroutine:
$
	\widetilde \pi_{\lambda} \leftarrow \text{NPG}_{Sub}\left(\lambda, \pi, \eta, N \right)
$.
    \STATE Compute and output the approximate gradient at $\lambda$:
    $
	\widetilde \nabla D\left(\lambda\right) := U_{g}^{\widetilde \pi_{\lambda}}(\rho) -{b} $.
\end{algorithmic}
\end{algorithm}

{Below, we give a formal statement for the convergence result of Algorithm \ref{alg:bisection} (cf. Theorem \ref{thm:bisection convergence}).}
\begin{theorem}[Restatement of Theorem \ref{thm:bisection convergence}]\label{thm:bisection convergence app}
{Suppose that Assumptions \ref{assump:slater}, \ref{assump:strong duality}, and \ref{assump:initial distribution} hold. When $n= 1$, for every $\varepsilon$, $\varepsilon_1>0$, Algorithm \ref{alg:bisection} with the parameters
\begin{equation}
\eta = \frac{(1-\gamma)}{\tau},\
N_1 = \frac{1}{1-\gamma} \log \left(\frac{4 |\mathcal{A}| \left(1+C_{2}+\tau\log |\mathcal{A}|\right)}{(1-\gamma)^{3} \tau \varepsilon}\right), N_2 = \frac{1}{1-\gamma} \log \left(\frac{2\left(1+C_2+\tau\log|\mathcal{A}|\right)}{\varepsilon_1 \tau(1-\gamma)}\right),
\end{equation}
returns a solution $(\pi, \lambda)$ in at most $\log_2 \left(\ell C_2/\varepsilon\right)$ outer loops, such that
\begin{subequations}
\begin{align}
D(\lambda) - D_\tau^\star &\leq \varepsilon_0,\label{eq:thm2 results_app_0} \\
\left\|\pi - \pi_\tau^\star\right\|_2 &\leq C_1\sqrt{\varepsilon_0} + \varepsilon_1,\label{eq:thm2 results_app_1}\\
\left|V_\tau^{\pi}(\rho) - V_\tau^\star\right| &\leq 2\varepsilon_0 + \ell_cC_1C_2\sqrt{\varepsilon_0} + \left(\ell_cC_2 + \frac{3\gamma}{2\tau}\right)\varepsilon_1,\label{eq:thm2 results_app_3}\\
 \left[b - U_{g}^{\pi}(\rho) \right]_+ &\leq \ell_c\left(C_1\sqrt{\varepsilon} +\varepsilon_1 \right),\label{eq:thm2 results_app_2}
\end{align}
\end{subequations}
where 
\begin{equation}
\varepsilon_0 = \frac{3C_2}{2}\varepsilon,\ \ell = \frac{2\ln 2 \left(|\mathcal{A}|+(1-\gamma)^2\sqrt{|\mathcal{A}|}\right)}{\tau (1-\gamma)^3 d},
\ \ell_c = \frac{\sqrt{|\mathcal{A}|}}{(1-\gamma)^2},\ C_1 = \sqrt{\frac{2 (1-\gamma)\ln 2}{\tau d }},\ C_2 = \frac{V_\tau^\star-V_{\tau}^{\overline{\pi}}(\rho)}{\xi}.
\end{equation}
The total iteration complexity is $\log_2 \left(\ell C_2/\varepsilon\right) \times N_2+N_3 = \mathcal{O}\left(\log^2(1/\varepsilon) + \log(1/\varepsilon_1)\right)$ with primal error bounds $\mathcal{O}\left(\sqrt{\varepsilon}+\varepsilon_1\right)$ given by (\ref{eq:thm2 results_app_1}) - (\ref{eq:thm2 results_app_2}) and a dual error bound $\mathcal{O}(\varepsilon)$ given by (\ref{eq:thm2 results_app_0}).}
\end{theorem}
We remark that the constants $\ell$, $\ell _c$, $C_1$, and $C_2$ used in Theorem \ref{thm:bisection convergence app} coincide with those used in previous sections, except that they correspond to the 1-dimensional situation.
The maximum number of iterations $N_1$ and $N_2$ in Theorem \ref{thm:bisection convergence app}, respectively, correspond to the $N_2$ and $N_3$ in Theorem \ref{thm:convergence_app}.
\begin{proof}
{Under the assumption that $\nabla D(0)< 0$ and $\nabla D(\xi)> 0$, the optimal dual variable $\lambda^\star$ must belong to the interval $(0,C_2)$ and $\nabla D(\lambda^\star) = 0$.
Denote $\widetilde{\nabla} D(\lambda) = \text{Grad}_{Sub}\left(\lambda,\pi,\eta,N_1\right)$, i.e. the output of the gradient estimator.
For a given threshold $\varepsilon$, when
\begin{equation}
N_1 = \frac{1}{1-\gamma} \log \left(\frac{4 |\mathcal{A}| \left(1+C_{2}+\tau\log |\mathcal{A}|\right)}{(1-\gamma)^{3} \tau \varepsilon}\right),
\end{equation}
it follows from the proof of Theorem \ref{thm:convergence} (cf. (\ref{eq:gradient estimation proof_app})) that
\begin{equation}
|\widetilde{\nabla} D({\lambda})-\nabla D({\lambda})|\leq \frac{\varepsilon}{2}.
\end{equation}
Thus, if $\widetilde{\nabla} D({\lambda})\geq \varepsilon$, we have that ${\nabla} D({\lambda})\geq \varepsilon/2$.
Similarly, if $\widetilde{\nabla} D({\lambda})\leq -\varepsilon$, we have that ${\nabla} D({\lambda})\leq -\varepsilon/2$.
Therefore, the lines \ref{alg_2:line 7}-\ref{alg_2:line 11} in Algorithm \ref{alg:bisection} shrink by a factor of 2 the search region that contains the optimal solution $\lambda^\star$.}

By leveraging the triangular inequality, we have that
\begin{equation}
|\widetilde{\nabla} D(\lambda)| \leq |\widetilde{\nabla} D(\lambda)-{\nabla} D(\lambda)|+|{\nabla} D(\lambda)-{\nabla} D(\lambda^\star)|
\leq \frac{\varepsilon}{2} + \ell \left|\lambda-\lambda^\star\right|.
\end{equation}
where we apply the smoothness of the dual function (cf. Proposition \ref{prop:dual smoothness}).
Thus, Algorithm \ref{alg:bisection} terminates in at most $t = \log_2 \left(\ell C_2/\varepsilon\right)$ iterations with an $\varepsilon$-optimal stationary point due to
\begin{equation}\label{eq:gradient output error_app}
|\widetilde{\nabla} D(\lambda)|\leq  \frac{\varepsilon}{2} + \ell \left|\lambda-\lambda^\star\right| \leq  \frac{\varepsilon}{2} + \ell \left(\frac{1}{2}\right)^tC_2\leq \varepsilon,
\end{equation}
where $\lambda$ denotes the midpoint $(p_t+q_t)/2$ generated in the $t$-th iteration in line \ref{alg_2:shrink}.

{Now, suppose that $\lambda$ is the output solution with $|\widetilde{\nabla} D(\lambda)|\leq \varepsilon$. 
It holds that
\begin{equation}\label{eq:dual function quality_app}
D(\lambda)-D^\star_\tau \leq \left|\nabla D(\lambda)\right| \times\left|\lambda-\lambda^\star\right| \leq\left( |\widetilde{\nabla} D(\lambda)| +\frac{\varepsilon}{2}\right)C_2 =\frac{3C_2}{2}\varepsilon,
\end{equation}
where we have used (\ref{eq:gradient error bound_124_app}) and (\ref{eq:gradient output error_app}).
By substituting (\ref{eq:dual function quality_app}) into the proof of Theorem \ref{thm:convergence} and letting $\varepsilon_0 = 3C_2\varepsilon/2$, the desired bounds follow.
The convergence is linear and the total iteration complexity is upper-bounded by $\log_2 \left(\ell C_2/\varepsilon\right) \times N_2+N_3 = \mathcal{O}(\log^2(1/\varepsilon) + \log(1/\varepsilon_1))$.}
\end{proof}

\begin{corollary}[Restatement of Corollary \ref{cor:for bisection}]\label{cor:original mdp solution2_app}
{Suppose that Assumptions \ref{assump:slater}, \ref{assump:strong duality}, and \ref{assump:initial distribution} hold. Let 
\begin{equation}
\tau=\frac{(1-\gamma) \varepsilon}{4 \log |\mathcal{A}|}.
\end{equation}
Then, Algorithm \ref{alg:bisection} 
computes a solution $\pi$ for the standard CMDP such that
\begin{subequations}
\begin{align}
\left|V^{\pi^\star}(\rho)-V^{\pi}(\rho)\right| & =\mc{O}{(\varepsilon)},\\
\left[b - U_{g}^{\pi}(\rho) \right]_+ &= \mathcal{O}(\varepsilon),
\end{align}
\end{subequations}
in ${\mathcal{O}}\left(\log^2(1/\varepsilon) \right)$ iterations, where $\pi^\star$ is an optimal policy to the standard CMDP.}
\end{corollary}
\begin{proof}
{The proof can be fully adopted from that of Corollary \ref{cor:original mdp solution_app}
The main difference lies in the outer-loop complexity. 
To have $\ell_cC_1C_2\sqrt{\varepsilon_0}  = \varepsilon/2$, compared to the previous $\widetilde{\mathcal{O}}\left(1/\varepsilon^2\right)$ total iterations, it only requires $\mathcal{O}\left(\log^2 (1/\varepsilon)\right)$ total iterations for Algorithm \ref{alg:bisection} when there is a single constraint.
}
\end{proof}

 \section{Discussion About Assumption \ref{assump:strong duality}}\label{sec:app_discuss}
In the paper, our analysis depends on Assumption \ref{assump:strong duality}, which posits that strong duality is valid for the entropy-regularized constrained Markov decision process (CMDP). It can be readily confirmed that the primal problem \eqref{prob:maxentrl} is tantamount to the maximin problem:

\begin{equation}\label{eq:maximin_prob}
\max_{\pi\in \Pi}\min_{\Lamb\geq 0}V_{\tau}^\pi(\rho) + \Lamb^\top\left(U_\mb{g}^\pi (\rho)-\mb{b}\right).
\end{equation}

Conversely, the dual approach employed in this paper essentially addresses the minimax problem:

\begin{equation}\label{eq:minimax_prob}
\min_{\Lamb\geq 0}\max_{\pi\in \Pi} V_{\tau}^\pi(\rho) + \Lamb^\top\left(U_\mb{g}^\pi (\rho)-\mb{b}\right) = \min_{\Lamb\geq 0} D(\Lamb).
\end{equation}

If strong duality is applicable, a primal-dual pair $\left(\pi_\tau^\star, \Lamb^\star\right)$ exists that simultaneously solves both the maximin problem \eqref{eq:maximin_prob} and the minimax problem \eqref{eq:minimax_prob}. Consequently, we can resolve the primal problem \eqref{prob:maxentrl} by tackling the minimax problem \eqref{eq:minimax_prob}.

In the absence of strong duality, a non-zero duality gap exists between the optimal values of \eqref{eq:maximin_prob} and \eqref{eq:minimax_prob}, i.e.,
\begin{equation}
\min_{\Lamb\geq 0}\max_{\pi\in \Pi} V_{\tau}^\pi(\rho) + \Lamb^\top\left(U_\mb{g}^\pi (\rho)-\mb{b}\right)  > \max_{\pi\in \Pi}\min_{\Lamb\geq 0}V_{\tau}^\pi(\rho) + \Lamb^\top\left(U_\mb{g}^\pi (\rho)-\mb{b}\right).
\end{equation}
It is important to note that under Slater's condition (Assumption \ref{assump:slater}), the optimal dual variable for the minimax problem \eqref{eq:minimax_prob} remains bounded. This observation can be supported by the following lower bound on the dual function:

\begin{equation}
D(\Lamb) = \max_{\pi\in \Pi} V_{\tau}^\pi(\rho) + \Lamb^\top\left(U_\mb{g}^\pi (\rho)-\mb{b}\right)\geq \max_{\pi\in \Pi}\Lamb^\top\left(U_\mb{g}^\pi (\rho)-\mb{b}\right) \geq \Lamb^\top\boldsymbol \xi.
\end{equation}

Consequently, as $\Lamb\rightarrow \infty$, we have $D(\Lamb)\rightarrow \infty$, which implies that the optimal dual variable $\Lamb^\star := \argmin_{\Lamb\geq 0} D(\Lamb)$ must be bounded. Thus, the algorithm developed in this paper can be employed to solve the dual problem \eqref{eq:minimax_prob}.

To upper-bound the duality gap, we can employ the following sandwich inequality:
\begin{equation}
\begin{aligned}
\underbrace{\left[\min_{\Lamb\geq 0}\max_{\pi\in \Pi} V^\pi(\rho) + \Lamb^\top\left(U_\mb{g}^\pi (\rho)-\mb{b}\right)\right]}_{v_1} + \max_{\pi\in \Pi} \tau\cdot \mathcal{H}(\rho, \pi)
&\geq \min_{\Lamb\geq 0}\max_{\pi\in \Pi}V^\pi(\rho) + \Lamb^\top\left(U_\mb{g}^\pi (\rho)-\mb{b}\right)+\tau\cdot \mathcal{H}(\rho, \pi)\\[-5mm]
&=\min_{\Lamb\geq 0}\max_{\pi\in \Pi} V_{\tau}^\pi(\rho) + \Lamb^\top\left(U_\mb{g}^\pi (\rho)-\mb{b}\right)\\
&>\max_{\pi\in \Pi}\min_{\Lamb\geq 0}V_{\tau}^\pi(\rho) + \Lamb^\top\left(U_\mb{g}^\pi (\rho)-\mb{b}\right)\\
&>\underbrace{\max_{\pi\in \Pi}\min_{\Lamb\geq 0}V^\pi(\rho) + \Lamb^\top\left(U_\mb{g}^\pi (\rho)-\mb{b}\right)}_{v_2},
\end{aligned}
\end{equation}
where the last inequality is a result of the strict positiveness of the discounted entropy. Since $V^\pi(\rho) + \Lamb^\top\left(U_\mb{g}^\pi (\rho)-\mb{b}\right)$ can be considered the Lagrangian function of a standard CMDP, strong duality holds \citep{paternain2019safe}, and the terms $v_1$ and $v_2$ are equal. Using a sandwich bound, we deduce that the duality gap does not exceed $\max_{\pi\in \Pi} \tau\cdot \mathcal{H}(\rho, \pi) \leq {{\tau} \log |\mathcal{A}|}/({1-\gamma})$.
Hence, when the regularization weight $\tau$ is small, the duality gap is also small. By selecting $\tau = \mc{O}(\epsilon)$, the optimal solution to the minimax problem, denoted as $(\pi^\star, \Lambda^\star)$, offers an $\mc{O}(\epsilon)$ approximation to the maximin problem. This solution can serve as a warm-up for solving the primal problem.

This work primarily aims to address the entropy-regularized CMDP problem. However, we also show that if the weight $\tau$ is small, the solution to the entropy-regularized CMDP is a good approximation of the standard CMDP problem. 
As an extension, one can consider the discounted entropy as a primal regularizer to be incorporated into the augmented Lagrangian function. This approach allows for the direct analysis of convergence properties for the standard CMDP problem. For instance, \citet{li2021faster} explore smoothing the Lagrangian function using both the discounted entropy and a dual regularizer.

\section{Supporting Lemmas}
The followings are standard results about unregularized and entropy-regularized MDPs. We refer the reader to \citep{agarwal2021theory, mei2020global} for the proofs.
\begin{lemma}[Policy gradient for direct parameterization]\label{lemma:policy gradient_direct}
{Suppose that $V^\pi(\rho)$ is an unregularized value function. For the direct policy parameterization where ${\theta({s, a})=\pi_{\theta}(a \vert s)}$, the gradient is
\begin{equation}
\frac{\partial V^{\pi}(\rho)}{\partial \pi(a \vert s)}=\frac{1}{1-\gamma} d_{\rho}^{\pi}(s) Q^{\pi}(s, a).
\end{equation}}
\end{lemma}

\begin{lemma}[Policy gradient for soft-max parameterization]\label{lemma:policy gradient_softmax}
{Suppose that $V^\pi(\rho)$ is an unregularized value function. For the soft-max policy parameterization (cf. (\ref{eq:softmax})), the gradient is
\begin{equation}
\frac{\partial V^{\pi_{\theta}}(\rho)}{\partial \theta{(s, a)}}=\frac{1}{1-\gamma} d_{\rho}^{\pi_{\theta}}(s) \pi_{\theta}(a \vert s) A^{\pi_{\theta}}(s, a).
\end{equation}}
\end{lemma}

\begin{lemma}[Performance difference]\label{lemma:performance difference}
{Suppose that $V^\pi(\rho)$ is an unregularized value function. For all policies $\pi$ and $\pi^\prime$, it holds that
\begin{equation}
V^{\pi^{\prime}}(\rho)-V^{\pi}(\rho)=\frac{1}{1-\gamma} \sum_{s\in\mc{S}} d_{\rho}^{\pi}(s) \sum_{a\in\mc{A}}\left(\pi^{\prime}(a \vert s)-\pi(a \vert s)\right) \cdot Q^{\pi^{\prime}}(s, a).
\end{equation}}
\end{lemma}

\begin{lemma}[Soft sub-optimality]\label{lemma:soft sub-optimality}
{Suppose that $V_\tau^\pi(\rho)$ is an entropy-regularized value function and $\pi^\star_\tau$ is the optimal policy. For every policy $\pi$, it holds that
\begin{equation}
{V}^{\pi^\star}_{\tau}(\rho)-{V}_\tau^{\pi}(\rho)=\frac{\tau}{1-\gamma} \sum_{s\in\mc{S}}d_{\rho}^{\pi}(s) D_{\mathrm{KL}}\left[\pi(\cdot \vert s) \mid \pi_{\tau}^{\star}(\cdot \vert s)\right],
\end{equation}
where $D_\text{KL} \left[P(\cdot) \mid Q(\cdot) \right] := \sum _{x} P(x)\left(\log P(x) -\log Q(x)\right)$ is the KL divergence between probability distributions $P(\cdot)$ and $Q(\cdot)$.}
\end{lemma}

\end{document}